\documentclass[11pt]{article}

\usepackage{microtype}
\usepackage{graphicx}
\usepackage{subfigure}
\usepackage{booktabs} 
\usepackage{fullpage}
\usepackage[margin=1in]{geometry}
\usepackage[usenames]{color}
\usepackage{graphicx}
\usepackage{amsmath}
\usepackage{amsfonts}
\usepackage{amssymb}
\usepackage{amsthm}
\usepackage{kky}
\usepackage{natbib}
\usepackage{xcolor, soul}
\usepackage{times}
\usepackage[utf8]{inputenc}
\usepackage{url}            
\usepackage{booktabs}       
\usepackage{amsfonts}       
\usepackage{nicefrac}       
\usepackage{microtype}      
\usepackage{graphicx}
\usepackage{color}
\usepackage{mathrsfs}
\usepackage{tabularx}
\usepackage{fancyvrb}
\usepackage{comment}
\usepackage{algorithm}
\usepackage{algorithmic}
\usepackage{float}
\usepackage{wrapfig}
\usepackage{mathtools}
\usepackage{thmtools}
\usepackage{thm-restate}
\usepackage{paralist} 
\usepackage{multirow}
\usepackage{scrextend}
\usepackage{enumitem}
\usepackage{fullpage}
\usepackage{quoting}
\usepackage{hyperref}
\usepackage{tabularx}
\usepackage{xcolor}
\usepackage{xspace}
\usepackage{makecell}
\usepackage{adjustbox}

\def\shownotes{1}  
\ifnum\shownotes=1
\newcommand{\authnote}[2]{{$\ll$\textsf{\footnotesize #1 notes: #2}$\gg$}}
\else
\newcommand{\authnote}[2]{}
\fi

\newcommand{\dist}{\mathcal{D}}
\newcommand{\ellzo}{\ell_{0-1}}
\newcommand{\distp}{\mathcal{P}}
\newcommand{\distq}{\mathcal{Q}}
\DeclarePairedDelimiterX{\inp}[2]{\langle}{\rangle}{#1, #2}

\newcommand{\dtv}{d_{\mathrm{TV}}}
\newcommand{\langs}{\mathcal{L}}
\newcommand{\truetranslator}[2]{f_{#1\to #2}^*}

\newcommand{\err}{\mathrm{Err}}

\newcommand{\enc}[1]{\mathbf{E}_{#1}}
\newcommand{\dec}[1]{\mathbf{D}_{#1}}

\theoremstyle{definition}
\newtheorem{definition}{Definition}[section]

\newtheorem{theorem}{Theorem}[section]
\newtheorem{assumption}{Assumption}[section]

\newtheorem{proposition}{Proposition}[section]
\newtheorem{corollary}{Corollary}[section]

\makeatletter
\def\thm@space@setup{%
  \thm@preskip=\parskip \thm@postskip=0pt
}
\makeatother

\newcommand{\covering}{\mathcal{N}}
\newcommand{\Nat}{\mathbb{N}}
\newcommand{\Real}{\mathbb{R}}
\newcommand{\xx}{\mathbf{x}}
\newcommand{\zz}{\mathbf{z}}

\newcommand{\xxspace}{\mathcal{X}}
\newcommand{\zzspace}{\mathcal{Z}}
\newcommand{\Exp}{\mathbb{E}}
\newcommand{\risk}{\varepsilon}
\newcommand{\id}{\mathrm{id}}
\newcommand{\emprisk}{\widehat{\varepsilon}}
\newif\iflinear
\linearfalse

\title{On Learning Language-Invariant Representations \\ for Universal Machine Translation}
\author{Han Zhao\thanks{Carnegie Mellon University, han.zhao@cs.cmu.edu},~Junjie Hu\thanks{Carnegie Mellon University, junjieh@cs.cmu.edu},~Andrej Risteski\thanks{Carnegie Mellon University, aristesk@andrew.cmu.edu}} 

\begin{document}
\maketitle

\begin{abstract}
The goal of universal machine translation is to learn to translate between any pair of languages, given a corpus of paired translated documents for \emph{a small subset} of all pairs of languages. Despite impressive empirical results and an increasing interest in massively multilingual models, theoretical analysis on translation errors made by such universal machine translation models is only nascent. 

In this paper, we formally prove certain impossibilities of this endeavour in general, as well as prove positive results in the presence of additional (but natural) structure of data. 

For the former, we derive a lower bound on the translation error in the many-to-many translation setting, which shows that any algorithm aiming to learn shared sentence representations among multiple language pairs has to make a large translation error on at least one of the translation tasks, if no assumption on the structure of the languages is made. 

For the latter, we show that if the paired documents in the corpus follow a natural \emph{encoder-decoder} generative process, we can expect a natural notion of ``generalization'': a linear number of language pairs, rather than quadratic, suffices to learn a good representation. Our theory also explains what kinds of connection graphs between pairs of languages are better suited: ones with longer paths result in worse sample complexity in terms of the total number of documents per language pair needed. 

We believe our theoretical insights and implications contribute to the future algorithmic design of universal machine translation.
\end{abstract}

\section{Introduction} 
Despite impressive improvements in neural machine translation (NMT), training a large multilingual NMT model with hundreds of millions of parameters usually requires a collection of \emph{parallel corpora} at a large scale, on the order of millions or even billions of aligned sentences~\citep{johnson2017google,arivazhagan2019massively} for supervised training. Although it is possible to automatically crawl the web~\citep{nie1999cross,resnik-1999-mining,resnik-smith-2003-web} to collect parallel sentences for high-resource language pairs such as German-English and Chinese-English, it is often infeasible or expensive to manually translate large amounts of documents for low-resource language pairs, e.g., Nepali-English, Sinhala-English~\citep{guzman-etal-2019-flores}. Much recent progress in low-resource machine translation, has been driven by the idea of \emph{universal machine translation} (UMT), also known as \emph{multilingual machine translation}~\citep{zoph-knight-2016-multi,johnson2017google,gu2018universal}, which aims at training one single NMT to translate between multiple source and target languages. Typical UMT models leverage either a single shared encoder or language-specific encoders to map all source languages to a shared space, and translate the source sentences to a target language by a decoder. Inspired by the idea of UMT, there has been a recent trend towards learning language-invariant embeddings for multiple source languages in a shared latent space, which eases the cross-lingual generalization from high-resource languages to low-resource languages on many tasks, e.g., parallel corpus mining~\citep{schwenk-2018-filtering,artetxe2019massively}, sentence classification~\citep{conneau-etal-2018-xnli}, cross-lingual information retrieval~\citep{litschko2018unsupervised}, and dependency parsing~\citep{kondratyuk-straka-2019-75}, just to name a few.

The idea of finding an abstract ``lingua franca'' is very intuitive and the empirical results are impressive, yet theoretical understanding of various aspects of universal machine translation is limited. In this paper, we particularly focus on two basic questions: 
\begin{enumerate}
    \item \emph{How can we measure the inherent tradeoff between the quality of translation and how language-invariant a representation is?}
    \item \emph{How many language pairs do we need aligned sentences for, to be able to translate between any pair of languages?}
\end{enumerate}

Toward answering the first question, we show that in a completely assumption-free setup on the languages and distribution of the data, it is impossible to avoid making a large translation error on at least one pair of the translation tasks. Informally we highlight our first theorem as follows, and provide the formal statements in Theorems~\ref{thm:lowerbound} and \ref{thm:manyone}.

\begin{theorem}[Impossibility, Informal] 
There exist a choice of distributions over documents from different languages, s.t.\ for any choice of maps from the language to a common representation, at least one of the translation pairs must incur a high cost. In addition, there is an inherent tradeoff between the translation quality and the degree of representation invariance w.r.t.\ languages: the better the language invariance, the higher the cost on at least one of the translation pairs.
\end{theorem}

To answer the second question, we show that under fairly mild generative assumptions on the aligned documents for the pairwise translations, it is possible to not only do well on all of the pairwise translations, but also be able to do so after \emph{only seeing} aligned documents of a \emph{linear} number of languages, rather than a \emph{quadratic} one. We summarize the second theorem as follows, and provide a formal statement in Theorem~\ref{thm:main}.

\begin{theorem}[Sample complexity, Informal] Under a generative model where the documents for each language are generated from a ``ground-truth'' encoder-decoder model, after seeing aligned documents for a \emph{linear} number of pairs of languages, we can learn encoders/decoders that perform well on any unseen language pair. 
\end{theorem}

\paragraph{Notation and Setup}
We first introduce the notation used throughout the paper and then briefly describe the problem setting of universal machine translation. 

We use $\langs$ to denote the set of all possible languages, e.g., $\displaystyle \{\text{English}, \text{French}, \text{German}, \text{Chinese}, \ldots\}$. For any language $L\in\langs$, we associate with $L$ an alphabet $\Sigma_L$ that contains all the symbols from $L$. Note that we assume $|\Sigma_L|<\infty$, $\forall L\in\langs$, but different languages could potentially share part of the alphabet. Given a language $L$, a sentence $x$ in $L$ is a sequence of symbols from $\Sigma_L$, and we denote $\Sigma_L^*$ as the set of all sentences generated from $\Sigma_L$. Note that since in principle different languages could share the same alphabet, to avoid ambiguity, for each language $L$, there is a unique token $\langle L\rangle\in\Sigma_L$ and $\langle L\rangle\not\in\Sigma_L', \forall L'\neq L$. The goal of the unique token $\langle L \rangle$ is used to denote the source sentence, and a sentence $x$ in $L$ will have a unique prefix $\langle L\rangle$ to indicate that $x\in\Sigma_L^*$. Also, in this manuscript we will use sentence and string interchangeably. 

Formally, let $\{L_i\}_{i\in[K]}$\footnote{We use $[K]$ to denote the set $\{0, 1, \ldots, K-1\}$.} be the set of $K$ source languages and $L\not\in\{L_i\}_{i\in[K]}$ be the target language we are interested in translating to. For a pair of languages $L$ and $L'$, we use $\dist_{L, L'}$ to denote the joint distribution over the parallel sentence pairs from $L$ and $L'$. Given this joint distribution, we also use $\dist_{L,L'}(L)$ to mean the marginal distribution over sentences from $L$. Likewise we use $\dist_{L,L'}(L')$ to denote the corresponding marginal distribution over sentences from $L'$. Finally, for two sets $A$ and $B$, we use $A\sqcup B$ to denote the disjoint union of $A$ and $B$. In particular, when $A$ and $B$ are disjoint, their disjoint union equals the usual set union, i.e., $A\sqcup B = A\bigcup B$.

\section{Related Work} 
\paragraph{Multilingual Machine Translation}
Early studies on multilingual machine translation mostly focused on pivot methods~\citep{och2001statistical,cohn-lapata-2007-machine,de2006catalan,utiyama-isahara-2007-comparison} that use one pivot language to connect the translation between ultimate source and target languages, and train two separate statistical translation models~\citep{koehn-etal-2003-statistical} to perform source-to-pivot and pivot-to-target translations. Since the successful application of encoder-decoder architectures in sequential tasks~\citep{NIPS2014_5346}, neural machine translation~\citep{bahdanau2014neural,gnmt} has made it feasible to jointly learn from parallel corpora in multiple language pairs, and perform translation to multiple languages by a single model. Existing studies have been proposed to explore different variants of encoder-decoder architectures by using separate encoders (decoders) for multiple source (target) languages~\citep{dong-etal-2015-multi,firat-etal-2016-multi,firat-etal-2016-zero,zoph-knight-2016-multi,platanios18emnlp} or sharing the weight of a single encoder (decoder) for all source (target) languages~\citep{ha2016toward,johnson2017google}. Recent advances of universal neural machine translation have also been applied to improve low-resource machine translation~\citep{neubig-hu-2018-rapid,gu2018universal,arivazhagan2019massively,aharoni-etal-2019-massively} and downstream NLP tasks~\citep{artetxe2019massively,schwenk-douze-2017-learning}. Despite the recent empirical success in the literature, theoretical understanding is only nascent. Our work takes a first step towards better understanding the limitation of existing approaches and proposes a sufficient generative assumption that guarantees the success of universal machine translation.

\paragraph{Invariant Representations}
The line of work on seeking a shared multilingual embedding space started from learning cross-lingual word embeddings from parallel corpora~\citep{gouws2015bilbowa,luong2015bilingual} or a bilingual dictionary~\citep{mikolov2013exploiting,faruqui-dyer-2014-improving,artetxe2017learning,conneau2018word}, and later extended to learning cross-lingual contextual representations~\citep{devlin-etal-2019-bert,lample2019cross,huang-etal-2019-unicoder,conneau2019unsupervised} from monolingual corpora. The idea of learning invariant representations is not unique in machine translation. In fact, similar ideas have already been used in other contexts, including domain adaptation~\citep{ganin2016domain,zhao2018adversarial,zhao2019learning,combes2020domain}, fair representations~\citep{zemel2013learning,zhang2018mitigating,zhao2019inherent,zhao2019conditional} and counterfactual reasoning in causal inference~\citep{johansson2016learning,shalit2017estimating}. Different from these existing work which mainly focuses on binary classification, our work provides the first impossibility theorem on learning language-invariant representations in terms of recovering a perfect translator under the setting of seq-to-seq learning.


\section{An Impossibility Theorem}
\label{sec:negative}
In this section, for the clarity of presentation, we first focus the deterministic setting where for each language pair $L\neq L'$, there exists a ground-truth translator $\truetranslator{L}{L'}:\Sigma_L^*\to\Sigma_{L'}^*$ that takes an input sentence $x$ from the source language $L$ and outputs the ground-truth translation $\truetranslator{L}{L'}(x)\in\Sigma_{L'}^*$. Later we shall extend the setup to allow a probabilistic extension as well. Before we proceed, we first describe some concepts that will be used in the discussion.

Given a feature map $g:\xxspace\to\zzspace$ that maps instances from the input space $\xxspace$ to feature space $\zzspace$, we define $g_\sharp\dist\defeq \dist\circ g^{-1}$ to be the induced (pushforward) distribution of $\dist$ under $g$, i.e., for any event $E'\subseteq\zzspace$, $\Pr_{g\sharp\dist}(E') \defeq \Pr_\dist(g^{-1}(E')) = \Pr_\dist(\{x\in\xxspace\mid g(x)\in E'\})$. For two distribution $\dist$ and $\dist'$ over the same sample space, we use the total variation distance to measure the discrepancy them: $\dtv(\dist, \dist')\defeq \sup_{E}|\Pr_\dist(E) - \Pr_{\dist'}(E)|$, where $E$ is taken over all the measurable events under the common sample space. We use $\ind(E)$ to denote the indicator function which takes value 1 iff the event $E$ is true otherwise 0.

In general, given two sentences $x$ and $x'$, we use $\ell(x, x')$ to denote the loss function used to measure their distance. For example, we could use a $0-1$ loss function $\ellzo(x, x') = 0$ iff $x = x'$ else 1. If both $x$ and $x'$ are embedded in the same Euclidean space, we could also use the squared loss $\ell_2(x, x')$ as a more refined measure. To measure the performance of a translator $f$ on a given language pair $L\to L'$ w.r.t.\ the ground-truth translator $\truetranslator{L}{L'}$, we define the error function of $f$ as 
\begin{align*}
    \err_\dist^{L\to L'}(f) \defeq&~ \Exp_\dist\left[\ellzo(f(X), \truetranslator{L}{L'}(X))\right],
\end{align*}
which is the translation error of $f$ as compared to the ground-truth translator $\truetranslator{L}{L'}$. For universal machine translation, the input string of the translator can be any sentence from any language. To this end, let $\Sigma^*$ be the union of all the sentences/strings from all the languages of interest: $\Sigma^*\defeq \bigcup_{L\in\langs}\Sigma^*_L$. Then a universal machine translator of target language $L\in\langs$ is a mapping $f_L:\Sigma^*\to\Sigma_L^*$. In words, $f_L$ takes as input a string (from one of the possible languages) and outputs the corresponding translation in target language $L$. It is not hard to see that for such task there exists a perfect translator $f_L^*$:
\begin{equation}
    f_L^*(x) = \sum_{L'\in\langs}\ind(x\in\Sigma_{L'}^*)\cdot\truetranslator{L'}{L}(x).
\label{equ:optimal}
\end{equation}
Note that $\{\Sigma^*_{L'}\mid L'\in\langs\}$ forms a partition of $\Sigma^*$, so exactly one of the indicator $\ind(x\in\Sigma_{L'}^*)$ in~\eqref{equ:optimal} will take value 1. 

Given a target language $L$, existing approaches for universal machine seek to find an intermediate space $\zzspace$, such that source sentences from different languages are aligned within $\zzspace$. In particular, for each source language $L'$, the goal is to find a feature mapping $g_{L'}:\Sigma^*_{L'}\to\zzspace$ so that the induced distributions of different languages are close in $\zzspace$. The next step is to construct a decoder $h:\zzspace\to\Sigma_L^*$ that maps feature representation in $\zzspace$ to sentence in the target language $L$.

One interesting question about the idea of learning language-invariant representations is that, whether such method will succeed even under the benign setting where there is a ground-truth universal translator and the learner has access to infinite amount of data with unbounded computational resources. That is, we are interested in understanding the information-theoretic limit of such methods for universal machine translation. 

In this section we first present an impossibility theorem in the restricted setting of translating from two source languages $L_0$ and $L_1$ to a target language $L$. Then we will use this lemma to prove a lower bound of the universal translation error in the general many-to-many setting. We will mainly discuss the implications and intuition of our theoretical results and use figures to help illustrate the high-level idea of the proof. 

\subsection{Two-to-One Translation}
Recall that for each translation task $L_i\to L$, we have a joint distribution $\dist_{L_i,L}$ (parallel corpora) over the aligned source-target sentences. For convenience of notation, we use $\dist_i$ to denote the marginal distribution $\dist_{L_i, L}(L_i)$, $\forall i\in[K]$ when the target language $L$ is clear from the context. Given a fixed constant $\epsilon > 0$, we first define the  \emph{$\epsilon$-universal language mapping}:
\begin{definition}[$\epsilon$-Universal Language Mapping]
\label{def:umap}
A map $g:\bigcup_{i\in[K]}~\Sigma_{L_i}^*\to\zzspace$ is called an $\epsilon$-universal language mapping if $\dtv(g_\sharp\dist_i, g_\sharp\dist_j) \leq \epsilon,~\forall i\neq j$.
\end{definition}
In particular, if $\epsilon = 0$, we call the corresponding feature mapping a universal language mapping. In other words, a universal language mapping perfectly aligns the feature representations of different languages in feature space $\zzspace$. The following lemma provides a useful tool to connect the 0-1 translation error and the TV distance between the corresponding distributions. 
\begin{restatable}{lemma}{keylemma}
\label{lemma:key}
Let $\Sigma\defeq\bigcup_{L\in\langs}\Sigma_L$ and $\dist_\Sigma$ be a language model over $\Sigma^*$. For any two string-to-string maps $f,f':\Sigma^*\to\Sigma^*$, let $f_\sharp\dist_\Sigma$ and $f'_\sharp\dist_\Sigma$ be the corresponding pushforward distributions. Then $\dtv(f_\sharp\dist_\Sigma, f'_\sharp\dist_\Sigma)\leq \Pr_{\dist_\Sigma}(f(X)\neq f'(X))$ where $X\sim\dist_\Sigma$. 
\end{restatable}
\begin{proof}
Note that the sample space $\Sigma^*$ is countable. For any two distributions $\distp$ and $\distq$ over $\Sigma^*$, it is a well-known fact that $\dtv(\distp, \distq) = \frac{1}{2}\sum_{y\in\Sigma^*}|\distp(y) - \distq(y)|$. Using this fact, we have:
\begin{align*}
    \dtv(f_\sharp\dist, f'_\sharp\dist) &= \frac{1}{2}\sum_{y\in\Sigma^*} \left|f_\sharp\dist(y) - f'_\sharp\dist(y)\right| \\
    &= \frac{1}{2}\sum_{y\in\Sigma^*} \left|\Pr_\dist(f(X) = y) - \Pr_\dist(f'(X) = y)\right| \\
    &= \frac{1}{2}\sum_{y\in\Sigma^*} \left|\Exp_\dist[\ind(f(X) = y)] - \Exp_\dist[\ind(f'(X) = y)]\right| \\
    &\leq \frac{1}{2}\sum_{y\in\Sigma^*} \Exp_\dist\left[\left|\ind(f(X) = y) - \ind(f'(X) = y)\right|\right] \\
    &= \frac{1}{2}\sum_{y\in\Sigma^*} \Exp_\dist[\ind(f(X) = y, f'(X)\neq y) + \ind(f(X) \neq y, f'(X)= y)] \\
    &= \frac{1}{2}\sum_{y\in\Sigma^*} \Exp_\dist\left[\ind(f(X) = y, f'(X)\neq f(X))\right] + \Exp_\dist\left[\ind(f'(X) = y, f'(X)\neq f(X))\right] \\
    &= \sum_{y\in\Sigma^*}\Exp_\dist\left[\ind(f(X) = y, f'(X)\neq f(X))\right] \\ 
    &= \sum_{y\in\Sigma^*}\Pr_\dist\left(f(X) = y, f'(X)\neq f(X)\right)  \\
    &= \Pr_\dist(f(X)\neq f'(X)).
\end{align*}
The second equality holds by the definition of the pushforward distribution. The inequality on the fourth line holds due to the triangule inequality and the equality on the seventh line is due to the symmetry between $f(X)$ and $f'(X)$. The last equality holds by the total law of probability. 
\end{proof}
The next lemma follows from the data-processing inequality for total variation and it shows that if languages are close in a feature space, then any decoder cannot increase the corresponding discrepancy of these two languages in the output space.

\begin{restatable}{lemma}{monotonous}(Data-processing inequality)
Let $\dist$ and $\dist'$ be any distributions over $\zzspace$, then for any decoder $h:\zzspace\to\Sigma_L^*$, $\dtv(h_\sharp\dist, h_\sharp\dist')\leq \dtv(\dist, \dist')$.
\end{restatable}
As a direct corollary, this implies that any distributions induced by a decoder over $\epsilon$-universal language mapping must also be close in the output space:
\begin{corollary}
\label{coro:simple}
If $g:\Sigma^*\to\zzspace$ is an $\epsilon$-universal language mapping, then for any decoder $h:\zzspace\to\Sigma_L^*$, $\dtv((h\circ g)_\sharp\dist_0, (h\circ g)_\sharp\dist_1)\leq \epsilon$.
\end{corollary}
\begin{figure}[tb]
    \centering
    \includegraphics[width=0.8\linewidth]{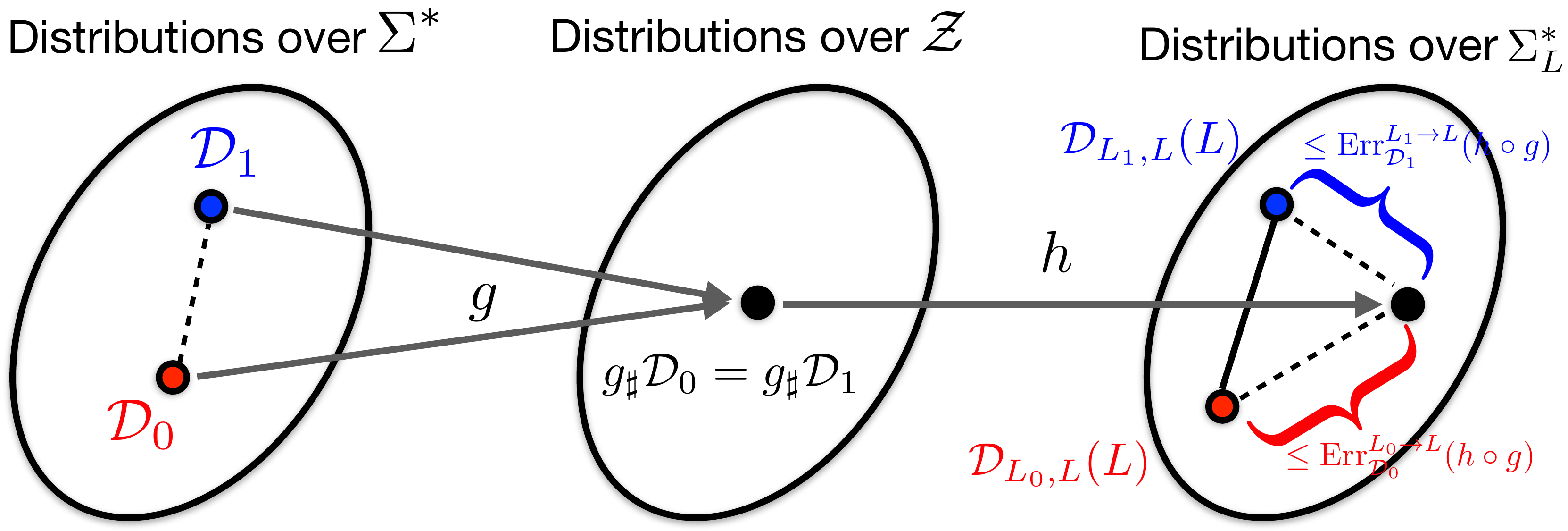}
    \caption{Proof by picture: Language-invariant representation $g$ induces the same feature distribution over $\zzspace$, which leads to the same output distribution over the target language $\Sigma_L^*$. However, the parallel corpora of the two translation tasks in general have different marginal distributions over the target language, hence a triangle inequality over the output distributions gives the desired lower bound.}
    \label{fig:proof1}
\end{figure}
With the above tools, we can state the following theorem that characterizes the translation error in a two-to-one setting:
\begin{restatable}{theorem}{basic}(Lower bound, Two-to-One)
\label{thm:lowerbound}
Consider a restricted setting of universal machine translation task with two source languages where $\Sigma^* = \Sigma_{L_0}^*\bigcup \Sigma_{L_1}^*$ and the target language is $L$. Let $g:\Sigma^*\to\zzspace$ be an $\epsilon$-universal language mapping, then for any decoder $h:\zzspace\to\Sigma_L^*$, we have
\begin{align}
    \err_{\dist_{0}}^{L_0\to L}(h\circ g) + \err_{\dist_{1}}^{L_1\to L}(h\circ g) \geq \dtv(\dist_{L_0, L}(L), \dist_{L_1, L}(L)) - \epsilon.
\label{equ:lowerbound}
\end{align}
\end{restatable}
\begin{proof}[Proof of Theorem~\ref{thm:lowerbound}]
First, realize that $\dtv(\cdot, \cdot)$ is a distance metric, the following chain of triangle inequalities hold:
\begin{align*}
    \dtv(\dist_{L_0, L}(L), &~\dist_{L_1, L}(L))\leq \dtv(\dist_{L_0, L}(L), (h\circ g)_\sharp\dist_0) \\
    &+ \dtv((h\circ g)_\sharp\dist_1, \dist_{L_1, L}(L)) \\
    &+ \dtv((h\circ g)_\sharp\dist_0, (h\circ g)_\sharp\dist_1).
\end{align*}
Now by the assumption that $g$ is an $\epsilon$-universal language mapping and Corollary~\ref{coro:simple}, the third term on the RHS of the above inequality, $\dtv((h\circ g)_\sharp\dist_0, (h\circ g)_\sharp\dist_1)$, is upper bounded by $\epsilon$. Furthermore, note that since the following equality holds:
\begin{equation*}
    \dist_{L_i, L}(L) = {\truetranslator{L_i}{L}}_\sharp\dist_i,\quad\forall i\in\{0, 1\}, 
\end{equation*}
we can further simplify the above inequality as
\begin{align*}
    \dtv(\dist_{L_0, L}(L), \dist_{L_1, L}(L)) \leq \dtv({\truetranslator{L_0}{L}}_\sharp\dist_0, (h\circ g)_\sharp\dist_0) + \dtv((h\circ g)_\sharp\dist_1, {\truetranslator{L_1}{L}}_\sharp\dist_1) + \epsilon. 
\end{align*}
Now invoke Lemma~\ref{lemma:key} for $i \in\{0, 1\}$ to upper bound the first two terms on the RHS, yielding:
\begin{equation*}
    \dtv({\truetranslator{L_i}{L}}_\sharp\dist_i, (h\circ g)_\sharp\dist_i) \leq \Pr_{\dist_i}\left((h\circ g)(X)\neq \truetranslator{L_i}{L}(X)\right) = \err_{\dist_{i}}^{L_i\to L}(h\circ g).
\end{equation*}
A simple rearranging then completes the proof.
\end{proof}

\paragraph{Remark}
Recall that under our setting, there exists a perfect translator $f_L^*:\Sigma^*\to\Sigma_L^*$ in~\eqref{equ:optimal} that achieves zero translation error on both translation tasks. Nevertheless, the lower bound in Theorem~\ref{thm:lowerbound} shows that one cannot hope to simultaneously minimize the joint translation error on both tasks through universal language mapping. Second, the lower bound is algorithm-independent and it holds even with unbounded computation and data. Third, the lower bound also holds even if all the data are perfect, in the sense that all the data are sampled from the perfect translator on each task. Hence, the above result could be interpreted as a kind of \emph{uncertainty principle} in the context of universal machine translation, which says that any decoder based on language-invariant representations has to achieve a large translation error on at least one pair of translation task. We provide a proof-by-picture in Fig.~\ref{fig:proof1} to illustrate the main idea underlying the proof of Theorem~\ref{thm:lowerbound} in the special case where $\epsilon = 0$.

The lower bound is large whenever the distribution over target sentences differ between these two translation tasks. This often happens in practical scenarios where the parallel corpus of high-resource language pair contains texts over a diverse domain whereas as a comparison, parallel corpus of low-resource language pair only contains target translations from a specific domain, e.g., sports, news, product reviews, etc. Such negative impact on translation quality due to domain mismatch between source and target sentences has also recently been observed and confirmed in practical universal machine translation systems, see~\citet{shen2019source} and~\citet{pires2019multilingual} for more empirical corroborations. 

\subsection{Many-to-Many Translation}
Theorem~\ref{thm:lowerbound} presents a negative result in the setting where we have two source languages and one target language for translation. Nevertheless universal machine translation systems often involve multiple input and output languages simultaneously~\citep{gnmt,fair2019umt,artetxe2019massively,johnson2017google}. In this section we shall extend the previous lower bound in the simple two-to-one setting to the more general translation task of many-to-many setting. 

To enable such extension, i.e., to be able to make use of multilingual data within a single system, we need to modify the input sentence to introduce the language token $\langle L\rangle$ at the beginning of the input sentence to indicate the target language $L$ the model should translate to. This simple modification has already been proposed and used in practical MT systems~\citep[Section 3]{johnson2017google}. To give an example, consider the following English sentence to be translated to French, 
\begin{quote}
\emph{
$\langle \text{English}\rangle$ Hello, how are you?}    
\end{quote}
It will be modified to:
\begin{quote}
\emph{
$\langle \text{French}\rangle\langle \text{English}\rangle$ Hello, how are you?}    
\end{quote}
Note that the first token is used to indicate the target language to translate to while the second one is used to indicate the source language to avoid the ambiguity due to the potential overlapping alphabets between different languages.

Recall in Definition~\ref{def:umap} we define a language map $g$ to be $\epsilon$-universal iff $\dtv(g_\sharp\dist_i, g_\sharp\dist_j)\leq \epsilon$, $\forall i, j$. This definition is too stringent in the many-to-many translation setting since this will imply that the feature representations lose the information about which target language to translate to. In what follows we shall first provide a relaxed definition of $\epsilon$-universal language mapping in the many-to-many setting and then show that even under this relaxed definition, learning universal machine translator via language-invariant representations is impossible in the worst case. 
\begin{definition}[$\epsilon$-Universal Language Mapping, Many-to-Many]
\label{def:umapmm}
Let $\dist_{L_i, L_k}$, $i, k\in[K]$ be the joint distribution of sentences (parallel corpus) in translating from $L_i$ to $L_k$. 
A map $g:\bigcup_{i\in[K]}~\Sigma_{L_i}^*\to\zzspace$ is called an $\epsilon$-universal language mapping if there exists a partition of $\zzspace = \sqcup_{k\in [K]}\zzspace_k$ such that $\forall k\in[K]$ and $\forall i\neq j$, $g_\sharp\dist_{L_i, L_k}(L_i)$ and $g_\sharp\dist_{L_j, L_k}(L_j)$ are supported on $\zzspace_k$ and $\dtv(g_\sharp\dist_{L_i, L_k}(L_i), g_\sharp\dist_{L_j, L_k}(L_j)) \leq \epsilon$.
\end{definition}
First of all, it is clear that when there is only one target language, then Definition~\ref{def:umapmm} reduces to Definition~\ref{def:umap}. Next, the partition of the feature space $\zzspace = \sqcup_{k\in [K]}\zzspace_k$ essentially serves as a way to determine the target language $L$ the model should translate to. Note that it is important here to enforce the partitioning condition of the feature space $\zzspace$, otherwise there will be ambiguity in determining the target language to translate to. For example, if the following two input sentences 
\begin{quote}
\emph{$\langle \text{French}\rangle\langle \text{English}\rangle$ Hello, how are you?}    \\
\emph{$\langle \text{Chinese}\rangle\langle \text{English}\rangle$ Hello, how are you?}
\end{quote}
are mapped to the same feature representation $z\in\zzspace$, then it is not clear whether the decoder $h$ should translate $z$ to French or Chinese. 

With the above extensions, now we are ready to present the following theorem which gives a lower bound for both the maximum error as well as the average error in the many-to-many universal translation setting.
\begin{restatable}{theorem}{manyone}(Lower bound, Many-to-Many)
\label{thm:manyone}
Consider a universal machine translation task where $\Sigma^* = \bigcup_{i\in[K]} \Sigma_{L_i}^*$. Let $\dist_{L_i, L_k}$, $i, k\in[K]$ be the joint distribution of sentences (parallel corpus) in translating from $L_i$ to $L_k$. If $g:\Sigma^*\to\zzspace$ be an $\epsilon$-universal language mapping, then for any decoder $h:\zzspace\to\Sigma^*$, we have
\begin{align*}
    \max_{i,k\in[K]}~\err_{\dist_{L_i, L_k}}^{L_i\to L_k}(h\circ g) &\geq \frac{1}{2}\max_{k\in[K]}\max_{i\neq j}~\dtv(\dist_{L_i, L_k}(L_k), \dist_{L_j, L_k}(L_k)) - \frac{\epsilon}{2}, \\
    \frac{1}{K^2}\sum_{i,k\in[K]}~\err_{\dist_{L_i, L_k}}^{L_i\to L_k}(h\circ g) &\geq \frac{1}{K^2(K-1)}\sum_{k\in[K]}\sum_{i<j} \dtv(\dist_{L_i, L_k}(L_k), \dist_{L_j, L_k}(L_k)) - \frac{\epsilon}{2}.
\end{align*}
\end{restatable}
\begin{proof}[Proof of Theorem~\ref{thm:manyone}]
First let us fix a target language $L_k$. For each pair of source languages $L_i, L_j,~i\neq j$ translating to $L_k$, applying Theorem~\ref{thm:lowerbound} gives us:
\begin{equation}
    \err_{\dist_{L_i, L_k}}^{L_i\to L_k}(h\circ g) + \err_{\dist_{L_j, L_k}}^{L_j\to L_k}(h\circ g) \geq \dtv(\dist_{L_i, L_k}(L_k), \dist_{L_j, L_k}(L_k)) - \epsilon.
\label{equ:triangle}
\end{equation}
Now consider the pair of source languages $(L_{i^*}, L_{j^*})$ with the maximum $\dtv(\dist_{L_i, L_k}(L_k), \dist_{L_j, L_k}(L_k))$:
\begin{align}
    2\max_{i\in[K]}~\err_{\dist_{L_i, L_k}}^{L_i\to L_k}(h\circ g) &\geq \err_{\dist_{L_{i^*}, L_k}}^{L_{i^*}\to L_k}(h\circ g) + \err_{\dist_{L_{j^*}, L_k}}^{L_{j^*}\to L_k}(h\circ g) \nonumber\\
    &\geq \max_{i\neq j}~\dtv(\dist_{L_i, L_k}(L_k), \dist_{L_j, L_k}(L_k)) - \epsilon.
\label{equ:manymany}
\end{align}
Since the above lower bound~\eqref{equ:manymany} holds for any target language $L_k$, taking a maximum over the target languages yields:
\begin{equation*}
    2\max_{i,k\in[K]}~\err_{\dist_{L_i, L_k}}^{L_i\to L_k}(h\circ g) \geq \max_{k\in[K]}\max_{i\neq j}~\dtv(\dist_{L_i, L_k}(L_k), \dist_{L_j, L_k}(L_k)) - \epsilon,
\end{equation*}
which completes the first part of the proof. For the second part, again, for a fixed target language $L_k$, to lower bound the average error, we apply the triangle inequality in~\eqref{equ:triangle} iteratively for all pairs $i < j$, yielding:
\begin{equation*}
    (K-1)\sum_{i\in[K]}~\err_{\dist_{L_i, L_k}}^{L_i\to L_k}(h\circ g) \geq \sum_{i<j} \dtv(\dist_{L_i, L_k}(L_k), \dist_{L_j, L_k}(L_k)) - \frac{K(K-1)}{2}\epsilon.
\end{equation*}
Dividing both sides by $K(K-1)$ gives the average translation error to $L_k$. Now summing over all the possible target language $L_k$ yields:
\begin{equation*}
    \frac{1}{K^2}\sum_{i,k\in[K]}~\err_{\dist_{L_i, L_k}}^{L_i\to L_k}(h\circ g) \geq \frac{1}{K^2(K-1)}\sum_{k\in[K]}\sum_{i<j} \dtv(\dist_{L_i, L_k}(L_k), \dist_{L_j, L_k}(L_k)) - \frac{\epsilon}{2}.\qedhere
\end{equation*}
\end{proof}
It is clear from the proof above that both lower bounds in Theorem~\ref{thm:manyone} include the many-to-one setting as a special case. The proof of Theorem~\ref{thm:manyone} essentially applies the lower bound in Theorem~\ref{thm:lowerbound} iteratively. Again, the underlying reason for such negative result to hold in the worst case is due to the mismatch of distributions of the target language in different pairs of translation tasks. It should also be noted that the results in Theorem~\ref{thm:manyone} hold even if language-dependent encoders are used, as long as they induce invariant feature representations for the source languages. 

\paragraph{How to Bypass this Limitation?}
There are various ways to get around the limitations pointed out by the theorems in this section. 

One way is to allow the decoder $h$ to have access to the input sentences (besides the language-invariant representations) during the decoding process -- e.g. via an attention mechanism on the input level. Technically, such information flow from input sentences during decoding would break the Markov structure of ``input-representation-output'' in Fig.~\ref{fig:proof1}, which is an essential ingredient in the proof of Theorem~\ref{thm:lowerbound} and Theorem~\ref{thm:manyone}. Intuitively, in this case both language-invariant (hence language-independent) and language-dependent information would be used. 

Another way would be to assume extra structure on the distributions $\mathcal{D}_{L_i, L_j}$, i.e., by assuming some natural language generation process for the parallel corpora that are used for training (Cf.\ Section~\ref{sec:positive}). Since languages share a lot of semantic and syntactic characteristics, this would make a lot of sense --- and intuitively, this is what universal translation approaches are banking on. In the next section we will do exactly this --- we will show that under a suitable generative model, not only will there be a language-invariant representation, but it will be learnable using corpora from a very small (linear) number of pairs of language.

\section{Sample Complexity under a Generative Model}
\label{sec:positive}

The results from the prior sections showed that absent additional assumptions on the distributions of the sentences in the corpus, there is a fundamental limitation on learning language-invariant representations for universal machine translation. Note that our negative result also holds in the setting where there exists a ground-truth universal machine translator -- it's just that learning language-invariant representations cannot lead to the recovery of this ground-truth translator.  

In this section we show that with
additional natural structure on the distribution of the corpora we can resolve this issue. The structure is a natural underlying generative model from which sentences from different languages are generated, which ``models'' a common encoder-decoder structure that has been frequently used in practice~\citep{cho2014learning,NIPS2014_5346,ha2016toward}. Under this setting, we show that it is not only possible to learn the optimal translator, but it is possible to do so only seeing documents from only a small subset of all the possible language pairs. 

Moreover, we will formalize a notion of ``\emph{sample complexity}'' in terms of number of pairs of languages for which parallel corpora are necessary, and how it depends on the structure of the connection graph between language pairs. 

We first describe our generative model for languages and briefly talk about why such generative model could help to overcome the negative result in Theorem~\ref{thm:manyone}.

\subsection{Language Generation Process and Setup}
\paragraph{Language Generative Process}
\label{sec:deterministic}
The language generation process is illustrated in Fig.~\ref{fig:encdec}. Formally, we assume the existence of a shared ``semantic space'' $\zzspace$. Furthermore, for every language $L \in \langs$, we have a ``ground truth'' pair of encoder and decoder $(\enc{L}, \dec{L})$, where $\enc{L}: \Real^d \to \Real^d$, $\enc{L} \in \mathcal{F}$ is \emph{bijective} and $\dec{L} = \enc{L}^{-1}$. We assume that $\Fcal$ has a group structure under function composition: namely, for $\forall f_1, f_2\in\Fcal$, we have that $f_1^{-1}, f_2^{-1}\in\Fcal$ and $f_1\circ f_2^{-1}, f_2\circ f_1^{-1}\in \Fcal$ (e.g., a typical example of such group is the general linear group $\Fcal = GL_d(\RR)$).

\begin{figure}[tb]
    \centering
    \includegraphics[width=0.7\linewidth]{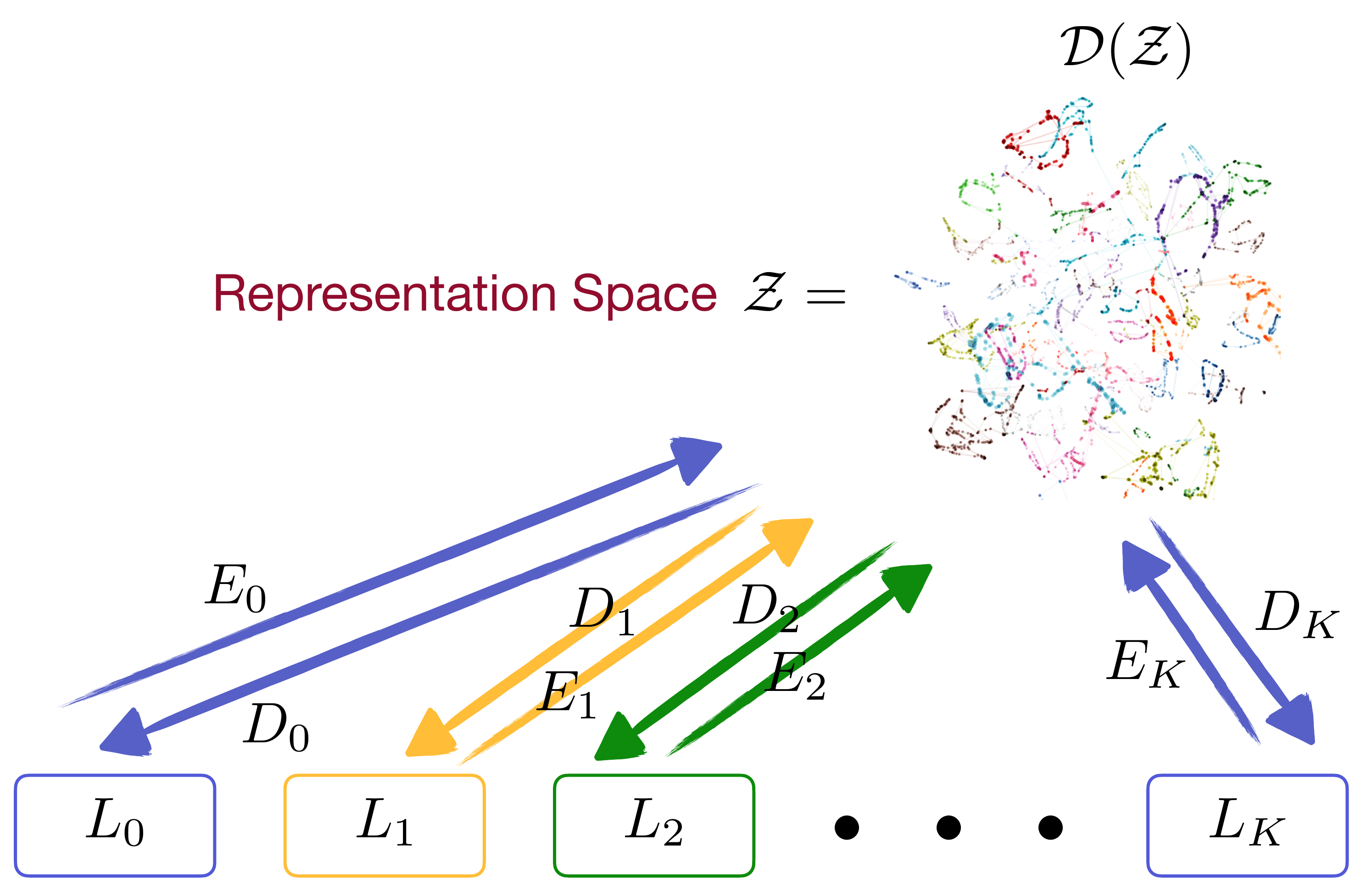}
    \caption{An encoder-decoder generative model of translation pairs. There is a global distribution $\dist$ over representation space $\zzspace$, from which sentences of language $L_i$ are generated via decoder $D_i$. Similarly, sentences could also be encoded via $E_i$ to $\zzspace$.}
    \label{fig:encdec}
\end{figure}
To generate a pair of aligned sentences for two languages $L,L'$, we first sample a $z \sim \dist$, and subsequently generate 
\begin{equation} 
x = \dec{L}(z), \quad x' = \dec{L'}(z),
\end{equation}
where the $x$ is a vector encoding of the appropriate sentence in $L$ (e.g., a typical encoding is a frequency count of the words in the sentence, or a sentence embedding using various neural network models~\citep{zhao2015self,kiros2015skip,wang2017sentence}). Similarly, $x'$ is the corresponding sentence in $L'$. Reciprocally, given a sentence $x$ from language $L$, the encoder $\enc{L}$ maps the sentence $x$ into its corresponding latent vector in $\zzspace$: $z = \enc{L}(x)$. 

We note that we assume this deterministic map between $z$ and $x$ for simplicity of exposition---in Section~\ref{s:randomized} we will extend the results to the setting where $x$ has a conditional distribution given $z$ of a parametric form.

We will assume the existence of a graph $H$ capturing the pairs of languages for which we have aligned corpora -- we can think of these as the ``high-resource'' pairs of languages. For each edge in this graph, we will have a corpus $S = \{(x_i, x'_i)\}_{i=1}^n$ of aligned sentences.\footnote{In general each edge can have different number of aligned sentences. We use the same number of aligned sentences $n$ just for the ease of presentation.}
The goal will be to learn encoder/decoders that perform well on the potentially unseen pairs of languages. To this end, we will be providing a sample complexity analysis for the number of paired sentences for each pair of languages with an edge in the graph, so we will need a measure of the complexity of $\Fcal$. 
We will use the covering number, though our proofs are flexible, and similar results would hold for Rademacher complexity, VC dimension, or any of the usual complexity measures. 

\begin{definition}[Covering number]
For any $\epsilon > 0$, the \emph{covering number} $\covering(\Fcal, \epsilon)$ of the function class $\Fcal$ under the $\ell_\infty$ norm is the minimum number $k\in\Nat$ such that $\Fcal$ could be covered with $k$ ($\ell_\infty$) balls of radius $\epsilon$, i.e., there exists $\{f_1, \ldots, f_k\}\subseteq\Fcal$ such that, for all $f\in\Fcal$, there exists $i\in[k]$ with $\|f - f_i\|_\infty = \max_{x\in\RR^d}\|f(x) - f_i(x)\|_2\leq\epsilon$.
\end{definition}

Finally, we will assume that the functions in $\Fcal$ are bounded and Lipschitz: 

\begin{assumption}[Smoothness and Boundedness]
$\Fcal$ is bounded under the $\|\cdot\|_\infty$ norm, i.e., there exists $M > 0$, such that $\forall f\in\Fcal$, $\|f\|_\infty\leq M$. Furthermore, there exists $0\leq \rho < \infty$, such that for $\forall x, x'\in\RR^d$, $\forall f\in\Fcal$, $\|f(x) - f(x')\|_2\leq \rho\cdot \|x - x'\|_2$.
\end{assumption}

\paragraph{Training Procedure} Turning to the training procedure, we will be learning encoders $E_L \in \Fcal$ for each language $L$. The decoder for that language will be $E^{-1}_L$, which is well defined since $\Fcal$ has a group structure. Since we are working with a vector space, rather than using the (crude) 0-1 distance, we will work with a more refined loss metric for a translation task $L\to L'$:
\begin{equation} 
\risk(E_L, E_{L'})\defeq \|E^{-1}_{L'}\circ E_L - \enc{L'}^{-1}\circ \enc{L}\|_{\ell_2({\dec{L}}_\sharp\dist)}^2.
\end{equation}
Note that the $\ell_2$ loss is taken over the distribution of the input samples ${\dec{L}}_\sharp\dist = {\enc{L}^{-1}}_\sharp\dist$, which is the natural one under our generative process. Again, the above error measures the discrepancy between the predicted translation w.r.t.\ the one give by the ground-truth translator, i.e., the composition of encoder $\enc{L}$ and decoder $\dec{L'}$. Straightforwardly, the empirical error over a corpus $S = \{(x_i, x'_i)\}_{i=1}^n$ of aligned sentences for a pair of languages $(L,L')$ is defined by
\begin{equation}
    \emprisk_S(E_L, E_{L'})\defeq \frac{1}{n}\sum_{i\in[n]}\|E^{-1}_{L'}\circ E_L(x_i) - x_i'\|_2^2,
\end{equation}
where $S$ is generated by the generation process. Following the paradigm of empirical risk minimization, the loss to train the encoders will be the obvious one: 
\begin{equation} 
\min_{\{E_L, L \in \langs\}}\quad\sum_{(L,L') \in H} \emprisk_S(E_L, E_{L'}).
\end{equation} 
\paragraph{Remarks} 
Before we proceed, one natural question to ask here is that, how does this generative model assumption circumvent the lower bound in Theorem~\ref{thm:manyone}? To answer this question, note the following easy proposition: 
\begin{proposition}
Under the encoder-decoder generative assumption, $\forall i, j\in[K]$, $\dtv(\dist_{L_i, L}(L), \dist_{L_j, L}(L)) = 0$.
\label{prop:zero}
\end{proposition}
Proposition~\ref{prop:zero} holds because the marginal distribution of the target language $L$ under any pair of translation task equals the pushforward of $\dist(Z)$ under $\dec{L}$: $\forall i\in[K], \dist_{L_i, L}(L) = {\dec{L}}_\sharp\dist(\zzspace)$. Hence the lower bounds gracefully reduce to 0 under our encoder-decoder generative process, meaning that there is no loss of translation accuracy using universal language mapping.  

 
\subsection{Main Result: Translation between Arbitrary Pairs of Languages} 
The main theorem we prove is that if the graph $H$ capturing the pairs of languages for which we have aligned corpora is connected, given sufficiently many sentences for each pair, we will learn encoder/decoders that perform well on the unseen pairs.
Moreover, we can characterize how good the translation will be based on the distance of the languages in the graph. Concretely: 
\begin{theorem}[Sample complexity under generative model] 
Suppose $H$ is connected. 
Furthermore, suppose the trained $\{E_L\}_{L \in \langs}$ satisfy 
\begin{equation*}
\forall L,L' \in H: \emprisk_S(E_L, E_{L'}) \leq \epsilon_{L,L'},    
\end{equation*}
for $\epsilon_{L,L'} > 0$. Furthermore, for $0 < \delta < 1$ suppose the number of sentences for each aligned corpora for each training pair $(L,L')$ is $\Omega\left(\frac{1}{\epsilon_{L,L'}^2}\cdot \left(\log\covering(\Fcal, \frac{\epsilon_{L,L'}}{16M}) + \log(K/\delta)\right)\right)$. Then, with probability $1-\delta$, for any pair of languages $(L,L') \in \langs\times\langs$ and $L = L_1, L_2, \dots, L_m = L'$ a path between $L$ and $L'$ in $H$, we have
$
  \varepsilon(E_L, E_{L'}) \leq 2\rho^2 \sum_{k=1}^{m-1} \epsilon_{L_k, L_{k+1}}.  
$
\label{thm:main} 
\end{theorem}
\paragraph{Remark}
We make several remarks about the theorem statement. Note that the guarantee is in terms of translation error rather than parameter recovery. In fact, due to the identifiability issue, we cannot hope to recover the ground truth encoders $\{\enc{L}\}_{L \in \langs}$: it is easy to see that composing all the encoders with an invertible mapping $f \in \Fcal$ and composing all the decoders with $f^{-1} \in \Fcal$ produces exactly the same outputs.

Furthermore, the upper bound is adaptive, in the sense that for any language pair $(L_i, L_j)$, the error depends on the sum of the errors connecting $(L_i, L_j)$ in the translation graph $H$. One can think naturally as the low-error edges as resource-rich pairs: if the function class $\Fcal$ is parametrized by finite-dimensional parameter space with dimension $p$, then using standard result on the covering number of finite-dimensional vector space~\citep{anthony2009neural}, we know that $\log\covering(\Fcal, \frac{\epsilon}{16M}) = \Theta(p\log(1/\epsilon))$; as a consequence, the number of documents needed for a pair scales as $\log(1/\epsilon_{L,L'})/\epsilon_{L,L'}^2$. 

Furthermore, as an immediate corollary of the theorem, if we assume $\epsilon_{L,L'} \leq \epsilon$ for all $(L,L') \in H$, we have $\varepsilon(E_L, E_{L'})\leq 2\rho^2 d_{L, L'}\cdot\epsilon$, where $d_{L, L'}$ is the length of the shortest path connecting $L$ and $L'$ in $H$. It also immediately follows that for any pair of languages $L, L'$, we have $\varepsilon(E_L, E_{L'}) \leq 2\rho^2 \diam(H)\cdot\epsilon$ where $\diam(H)$ is the diameter of $H$ -- thus the intuitive conclusion that graphs that do not have long paths are preferable. 

The upper bound in Theorem~\ref{thm:main} also provides a counterpoint to the lower-bound, showing that under a generative model for the data, it is possible to learn a pair of encoder/decoder for each language pair after seeing aligned corpora only for a linear number of pairs of languages (and not quadratic!), corresponding to those captured by the edges of the translation graph $H$. As a final note, we would like to point out that an analogous bound can be proved easily for other losses like the 0-1 loss or the general $\ell_p$ loss as well.
\begin{figure}[tb]
    \centering
    \includegraphics[width=0.5\linewidth]{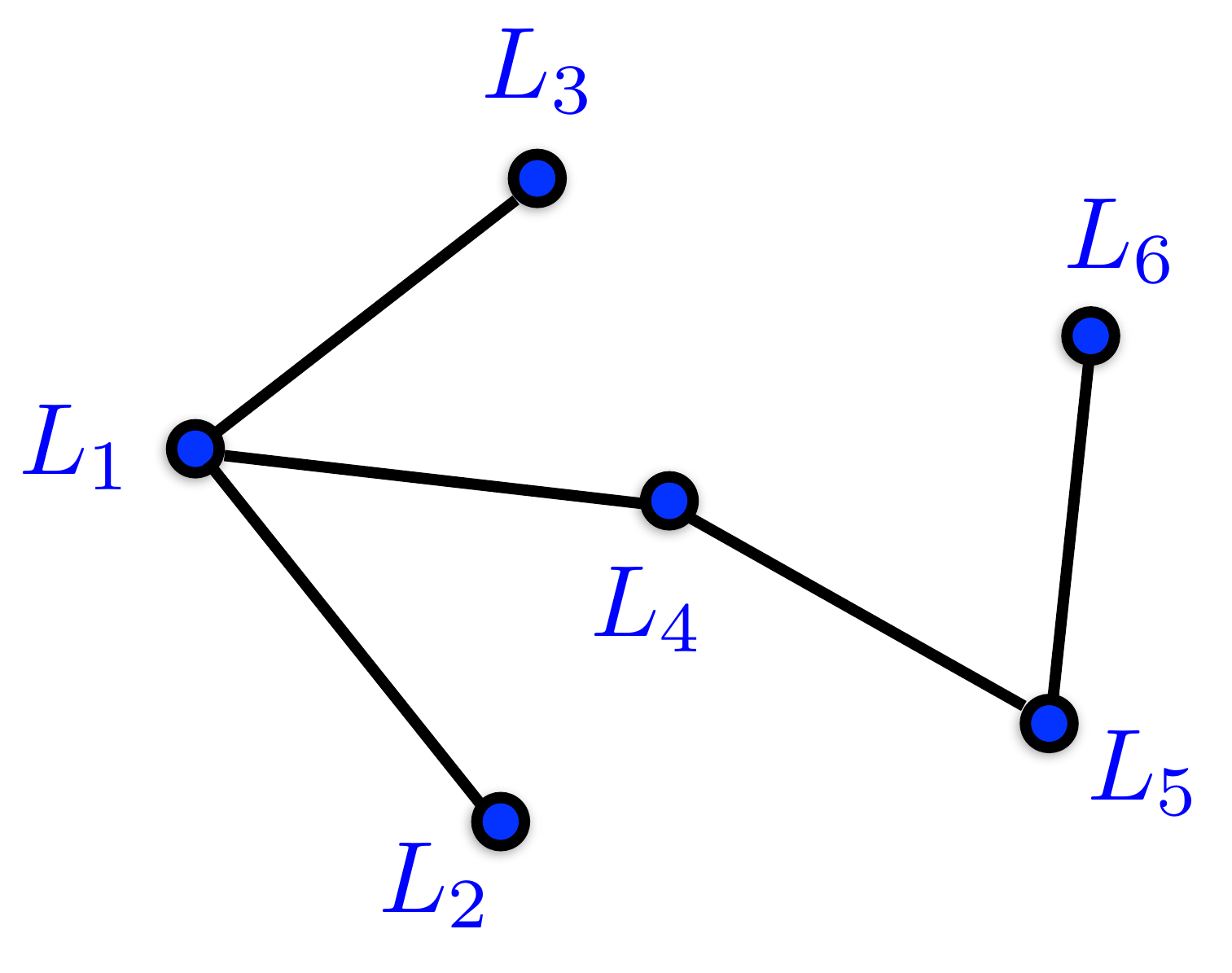}
    \caption{A translation graph $H$ over $K = 6$ languages. The existence of an edge between a pair of nodes $L_i$ and $L_j$ means that the learner has been trained on the corresponding language pair. In this example the diameter of the graph $\diam(H) = 4$: $L_3, L_1, L_4, L_5, L_6$.}
    \label{fig:tgraph}
\end{figure}

\subsection{Proof Sketch of the Theorem}
Before we provide the proof for the theorem, we first state several useful lemmas that will be used during our analysis. 
\paragraph{Concentration Bounds} 
The first step is to prove a concentration bound for the translation loss metric on each pair of languages. In this case, it will be easier to write the losses in terms of one single function: namely notice that in fact  
$\risk(E_L, E_L')$ only depends on $E^{-1}_{L'}\circ E_L$, and due to the group structure, $\mathcal{F} \ni f:=E^{-1}_{L'}\circ E_L$. To that end, we will abuse the notation somewhat and denote $\risk(f):= \risk(E_L, E_L')$. 
The following lemma is adapted from~\citet{bartlett1998thesamplecomplexityofp} where the bound is given in terms of binary classification error while here we present a bound using $\ell_2$ loss. At a high level, the bound uses covering number to concentrate the empirical loss metric to its corresponding population counterpart. 
\begin{restatable}{lemma}{ggap}
\label{thm:covering}
If $S = \{(x_i, x'_i)\}_{i=1}^n$ is sampled i.i.d.\ according to the encoder-decoder generative process, the following bound holds:
\begin{align*}
    \Pr_{S\sim \dist^n} \left(\sup_{f\in\Fcal} |\risk(f) - \emprisk_S(f)| \geq \epsilon\right) \leq 2\covering(\Fcal, \frac{\epsilon}{16M})\cdot \exp\left(\frac{-n\epsilon^2}{16M^4}\right).
\end{align*}
\end{restatable}
This lemma can be proved using a $\epsilon$-net argument with covering number. With this lemma, we can bound the error given by an empirical risk minimization algorithm:
\begin{restatable}{theorem}{gsingle}(Generalization, single task)
\label{thm:finite}
Let $S$ be a sample of size $n$ according to our generative process. Then for any $0< \delta < 1$, for any $f\in\Fcal$, w.p.\ at least $1-\delta$, the following bound holds:
\begin{equation}
    \risk(f) \leq \emprisk_S(f) + O\left(\sqrt{\frac{\log\covering(\Fcal, \frac{\epsilon}{16M}) + \log(1/\delta)}{n}}\right).
\label{equ:sbound}
\end{equation}
\end{restatable}
Theorem~\ref{thm:finite} is a finite sample bound for generalization on a single pair of languages. This bound gives us an error measure on an edge in the translation graph in Fig.~\ref{fig:tgraph}. Now, with an upper bound on the translation error of each \emph{seen} language pair, we are ready to prove the main theorem (Theorem~\ref{thm:main}) which bounds the translation error for all possible pairs of translation tasks:

\begin{proof}[Proof of Theorem \ref{thm:main}]
First, under the assumption of Theorem~\ref{thm:main}, for any pair of language $(L, L')$, we know that the corpus contains at least $\Omega\left(\frac{1}{\epsilon_{L,L'}^2}\cdot \left(\log\covering(\Fcal, \frac{\epsilon_{L,L'}}{16M}) + \log(K/\delta)\right)\right)$ parallel sentences. Then by Theorem \ref{thm:finite}, with probability $1-\delta$, for any $L,L'$ connected by an edge in $H$, we have
\begin{equation*}
    \risk(E_L, E_{L'}) \leq \emprisk(E_L, E_{L'}) + \epsilon_{L, L'} \leq \epsilon_{L, L'} + \epsilon_{L, L'} = 2\epsilon_{L, L'},
\end{equation*}
where the last inequality is due to the assumption that $\emprisk(E_L, E_{L'})\leq \epsilon_{L, L'}$. Now consider any $L,L' \in \langs\times\langs$, connected by a path 
\begin{equation*}
L' = L_1, L_2, L_3, \ldots, L_m = L
\end{equation*}
of length at most $m$. We will bound the error 
\begin{equation*}
\varepsilon(E_L,E_{L'}) = \|E^{-1}_{L'}\circ E_L - \mathbf{E}^{-1}_{L'}\circ \mathbf{E}_L\|_{\ell_2({\dec{L}}_\sharp\dist)}^2
\end{equation*}
by a judicious use of the triangle inequality. 
Namely, let's denote 
\begin{align*} 
I_1 &\defeq \mathbf{E}^{-1}_{L_1}\circ \mathbf{E}_{L_m},\\
I_{k}&\defeq E^{-1}_{L_1} \circ E_{L_k} \circ \mathbf{E}^{-1}_{L_k} \circ \mathbf{E}_{L_m},\qquad 2 \leq k \leq m-1,\\  
I_{m}&\defeq E_{L_1}^{-1} \circ  E_{L_m}.
\end{align*} 
Then, we can write 
\begin{align}
\|E^{-1}_{L'}\circ E_L - \enc{L'}^{-1}\circ \enc{L}\|_{\ell_2({\dec{L}}_\sharp\dist)} = \|\sum_{k=1}^{m-1} I_k-I_{k+1}\|_{\ell_2({\dec{L}}_\sharp\dist)} \leq \sum_{k=1}^{m-1} \|I_k-I_{k+1}\|_{\ell_2({\dec{L}}_\sharp\dist)}.
\label{eq:triangle}
\end{align}
Furthermore, notice that we can rewrite $I_k-I_{k+1}$ as 
\begin{align*}
E^{-1}_{L_1} \circ E_{L_k} \left(\enc{L_k}^{-1} \circ \enc{L_{k+1}} - E^{-1}_{L_k} \circ E_{L_{k+1}}\right) 
\mathbf{E}^{-1}_{L_{k+1}}\circ\mathbf{E}_{L_{m}}.
\end{align*}
Given that $E^{-1}_{L_1}$ and $E_{L_k}$ are $\rho$-Lipschitz we have
{\small
\begin{align*}
\left\|I_k-I_{k+1}\right\|_{\ell_2({\dec{L}}_\sharp\dist)} &= \left\|E^{-1}_{L_1} \circ E_{L_k} \left(E^{-1}_{L_k} \circ E_{L_{k+1}} - \mathbf{E}^{-1}_{L_k} \circ \mathbf{E}_{L_{k+1}}\right) 
\right\|_{\ell_2({\dec{L_{k+1}}}_\sharp\dist)} \\ 
&\leq \rho^2 \left\| \left(E^{-1}_{L_k} \circ E_{L_{k+1}} - \mathbf{E}^{-1}_{L_k} \circ \mathbf{E}_{L_{k+1}}\right) 
\right\|_{\ell_2({\dec{L_{k+1}}}_\sharp\dist)}\\
&\leq 2\rho^2\epsilon_{L_k,L_{k+1}},
\end{align*}}
where the first line is from the definition of pushforward distribution, the second line is due to the Lipschitzness of $\Fcal$ and the last line follows since all $(L_k, L_{k+1})$ are edges in $H$. Plugging this into \eqref{eq:triangle}, we have 
\begin{equation*}
  \|E^{-1}_{L'}\circ E_L - \enc{L'}^{-1}\circ \enc{L}\|_{\ell_2({\dec{L}}_\sharp\dist)} \leq 2\rho^2 \sum_{k=1}^m \epsilon_{L_k,L_{k+1}}.  
\end{equation*}
To complete the proof, realize that we need the events $|\risk(L_k, L_{k+1}) - \emprisk(L_k, L_{k+1})|\leq\epsilon_{L_k, L_{k+1}}$ to hold simultaneously for all the edges in the graph $H$. Hence it suffices if we can use a union bound to bound the failing probability. To this end, for each edge, we amplify the success probability by choosing the failure probability to be $\delta/K^2$, and we can then bound the overall failure probability as:
\begin{align*}
\Pr&\left(\text{At least one edge in the graph $H$ fails to satisfy~\eqref{equ:sbound}}\right) \\
&\leq \sum_{(i, j)\in H}\Pr\left(|\risk(L_i, L_j) - \emprisk(L_i, L_j)| > \epsilon_{L_i, L_j}\right) \\
&\leq \sum_{(i, j)\in H}\delta / K^2\\
&\leq \frac{K(K-1)}{2}\cdot\frac{\delta}{K^2} \\
&\leq \delta.
\end{align*}
The first inequality above is due to the union bound, and the second one is from Theorem~\ref{thm:finite} by choosing the failing probability to be $\delta / K^2$. 
\end{proof}

\subsection{Extension to Randomized Encoders and Decoders}
\label{s:randomized}
Our discussions so far on the sample complexity under the encoder-decoder generative process assume that the ground-truth encoders and decoders are \emph{deterministic} and \emph{bijective}. This might seem to be a quite restrictive assumption, but nevertheless our underlying proof strategy using transitions on the translation graph still works in more general settings. In this section we shall provide an extension of the previous deterministic encoder-decoder generative process to allow randomness in the generation process. Note that this extension simultaneously relaxes both the deterministic and bijective assumptions before. 

As a first step of the extension, since there is not a notion of inverse function anymore in the randomized setting, we first define the ground-truth encoder-decoder pair $(\enc{L}, \dec{L})$ for a language $L\in\langs$.
\begin{definition}
Let $\dist_r$ and $\dist_{r'}$ be two distributions over random seeds $r$ and $r'$ respectively. A randomized decoder $\dec{L_i}$ is a deterministic function that maps a feature $z$ along with a random seed $r$ to a sentences in language $L_i$. Similarly, a randomized encoder $\enc{L_i}$ maps a sentence $x\in\Sigma_{L_i}^*$ and a random seed $r'$ to a representation in $\zzspace$. $(\enc{L_i}, \dec{L_i})$ is called an encoder-decoder pair if it keeps the distribution $\dist$ over $\zzspace$ invariant under the randomness of $\dist_r$ and $\dist_{r'}$:
\begin{equation}
{\enc{L_i}}_\sharp({\dec{L_i}}_\sharp(\dist\times \dist_r)\times \dist_{r'}) = \dist,
\end{equation}
where we use $\dist\times \dist'$ to denote the product measure of distributions $\dist$ and $\dist'$.
\label{def:random}
\end{definition}
Just like the deterministic setting, here we still assume that $\enc{L_i}, \dec{L_i}\in\Fcal$ where $\Fcal$ is closed under function composition. Furthermore, in order to satisfy Definition~\ref{def:random}, we assume that $\forall~\dec{L_i}\in \Fcal$, there exists a corresponding $\enc{L_i}\in \Fcal$, such that $(\enc{L_i}, \dec{L_i})$ is an encoder-decoder pair that verifies Definition~\ref{def:random}. It is clear that the deterministic encoder-decoder pair in Section~\ref{sec:deterministic} is a special case of that in Definition~\ref{def:random}: in that case $\dec{L_i} = \enc{L_i}^{-1}$ so that $\enc{L_i}\circ \dec{L_i} = \id_\zzspace$, the identity map over feature space $\zzspace$. Furthermore there is no randomness from $r$ and $r'$, hence the invariant criterion becomes ${\enc{L_i}}_\sharp{\dec{L_i}}_\sharp\dist = ({\enc{L_i}}\circ {\dec{L_i}})_\sharp\dist = {\id_\zzspace}_\sharp\dist = \dist$, which trivially holds. 

The randomness mechanism in Definition~\ref{def:random} has several practical implementations in practice. For example, the denoising autoencoder~\citep{vincent2008extracting}, the encoder part of the conditional generative adversarial network~\citep{mirza2014conditional}, etc. Again, in the randomized setting we still need to have an assumption on the structure of $\Fcal$, but this time a relaxed one:
\begin{assumption}[Smoothness and Boundedness]
$\Fcal$ is bounded under the $\|\cdot\|_\infty$ norm, i.e., there exists $M > 0$, such that $\forall f\in\Fcal$, $\|f\|_\infty\leq M$. Furthermore, there exists $0\leq \rho < \infty$, such that for $\forall x, x'\in\RR^d$, $\forall f\in\Fcal$, $\|\Exp_{\dist_r}[f(x, r) - f(x', r)]\|_2\leq \rho\cdot \|x - x'\|_2$.
\end{assumption}
Correspondingly, we also need to slightly extend our loss metric under the randomized setting to the following:
\begin{equation*}
\small
\risk(E_L, D_{L'})\defeq \Exp_{r,r'}\|D_{L'}\circ E_L - \dec{L'}\circ \enc{L}\|_{\ell_2({\dec{L}}_\sharp(\dist\times \dist_r))}^2,
\end{equation*}
where the expectation is taken over the distributions over random seeds $r$ and $r'$. The empirical error could be extended in a similar way by replacing the population expectation with the empirical expectation. With the above extended definitions, now we are ready to state the following generalization theorem under randomized setting:
\begin{restatable}{theorem}{random}(Sample complexity under generative model, randomized setting)
Suppose $H$ is connected and the trained $\{E_L\}_{L \in \langs}$ satisfy 
\begin{equation*}
\forall L,L' \in H: \emprisk_S(E_L, D_{L'}) \leq \epsilon_{L,L'},    
\end{equation*}
for $\epsilon_{L,L'} > 0$. Furthermore, for $0 < \delta < 1$ suppose the number of sentences for each aligned corpora for each training pair $(L,L')$ is $\Omega\left(\frac{1}{\epsilon_{L,L'}^2}\cdot \left(\log\covering(\Fcal, \frac{\epsilon_{L,L'}}{16M}) + \log(K/\delta)\right)\right)$. Then, with probability $1-\delta$, for any pair of languages $(L,L') \in \langs\times\langs$ and $L = L_1, L_2, \dots, L_m = L'$ a path between $L$ and $L'$ in $H$, we have $\varepsilon(E_L, D_{L'}) \leq 2\rho^2 \sum_{k=1}^{m-1} \epsilon_{L_k, L_{k+1}}$.
\label{thm:random} 
\end{restatable}

We comment that Theorem~\ref{thm:random} is completely parallel to Theorem~\ref{thm:main}, except that we use generalized definitions under the randomized setting instead. Hence all the discussions before on Theorem~\ref{thm:main} also apply here.

\section{Discussion and Conclusion}
In this paper we provided the first theoretical study on using language-invariant representations for universal machine translation. Our results are two-fold. First, we showed that without appropriate assumption on the generative structure of languages, there is an inherent tradeoff between learning language-invariant representations versus achieving good translation performance jointly in general. In particular, our results show that if the distributions (language models) of the target language differ between different translation pairs, then any machine translation method based on learning language-invariant representations is bound to achieve a large error on at least one of the translation tasks, even with unbounded computational resources. 

On the positive side, we also show that, under appropriate generative model assumption of languages, e.g., a typical encoder-decoder model, it is not only possible to recover the ground-truth translator between any pair of languages that appear in the parallel corpora, but also we can hope to achieve a small translation error on sentences from unseen pair of languages, as long as they are connected in the so-called translation graph. This result holds in both deterministic and randomized settings. In addition, our result also characterizes how the relationship (distance) between these two languages in the graph affects the quality of translation in an intuitive manner: a graph with long connections results in a poorer translation.  

\section*{Acknowledgements}
We would like to thank Tom Mitchell for helpful conversations in the initial stages of the project, and Jiatao Gu for useful discussions on the recent progress in universal machine translation. HZ would like to acknowledge support from the DARPA XAI project, contract \#FA87501720152 and NVIDIA's GPU grant. JH is sponsored by the Air Force Research Laboratory under agreement number FA8750-19-2-0200.

\bibliography{reference.bib}

\begin{thebibliography}{60}
\providecommand{\natexlab}[1]{#1}
\providecommand{\url}[1]{\texttt{#1}}
\expandafter\ifx\csname urlstyle\endcsname\relax
  \providecommand{\doi}[1]{doi: #1}\else
  \providecommand{\doi}{doi: \begingroup \urlstyle{rm}\Url}\fi

\bibitem[Aharoni et~al.(2019)Aharoni, Johnson, and
  Firat]{aharoni-etal-2019-massively}
Aharoni, R., Johnson, M., and Firat, O.
\newblock Massively multilingual neural machine translation.
\newblock In \emph{Proceedings of the 2019 Conference of the North {A}merican
  Chapter of the Association for Computational Linguistics: Human Language
  Technologies, Volume 1 (Long and Short Papers)}, pp.\  3874--3884,
  Minneapolis, Minnesota, June 2019. Association for Computational Linguistics.
\newblock \doi{10.18653/v1/N19-1388}.

\bibitem[Anthony \& Bartlett(2009)Anthony and Bartlett]{anthony2009neural}
Anthony, M. and Bartlett, P.~L.
\newblock \emph{Neural network learning: Theoretical foundations}.
\newblock cambridge university press, 2009.

\bibitem[Arivazhagan et~al.(2019)Arivazhagan, Bapna, Firat, Lepikhin, Johnson,
  Krikun, Chen, Cao, Foster, Cherry, et~al.]{arivazhagan2019massively}
Arivazhagan, N., Bapna, A., Firat, O., Lepikhin, D., Johnson, M., Krikun, M.,
  Chen, M.~X., Cao, Y., Foster, G., Cherry, C., et~al.
\newblock Massively multilingual neural machine translation in the wild:
  Findings and challenges.
\newblock \emph{arXiv preprint arXiv:1907.05019}, 2019.

\bibitem[Artetxe \& Schwenk(2019)Artetxe and Schwenk]{artetxe2019massively}
Artetxe, M. and Schwenk, H.
\newblock Massively multilingual sentence embeddings for zero-shot
  cross-lingual transfer and beyond.
\newblock \emph{Transactions of the Association for Computational Linguistics},
  7:\penalty0 597--610, 2019.

\bibitem[Artetxe et~al.(2017)Artetxe, Labaka, and Agirre]{artetxe2017learning}
Artetxe, M., Labaka, G., and Agirre, E.
\newblock Learning bilingual word embeddings with (almost) no bilingual data.
\newblock In \emph{Proceedings of ACL 2017}, pp.\  451--462, 2017.

\bibitem[Bahdanau et~al.(2015)Bahdanau, Cho, and Bengio]{bahdanau2014neural}
Bahdanau, D., Cho, K., and Bengio, Y.
\newblock Neural machine translation by jointly learning to align and
  translate.
\newblock In \emph{International Conference on Learning Representations}, 2015.

\bibitem[Bartlett(1998)]{bartlett1998thesamplecomplexityofp}
Bartlett, P.
\newblock Thesamplecomplexityofp atternclassification withneuralnetworks:
  Thesizeoftheweightsismo reimportantthan thesizeofthenetwork.
\newblock \emph{IEEETrans. Inf. Theory}, 44\penalty0 (2), 1998.

\bibitem[Cho et~al.(2014)Cho, Van~Merri{\"e}nboer, Gulcehre, Bahdanau,
  Bougares, Schwenk, and Bengio]{cho2014learning}
Cho, K., Van~Merri{\"e}nboer, B., Gulcehre, C., Bahdanau, D., Bougares, F.,
  Schwenk, H., and Bengio, Y.
\newblock Learning phrase representations using rnn encoder-decoder for
  statistical machine translation.
\newblock \emph{arXiv preprint arXiv:1406.1078}, 2014.

\bibitem[Cohn \& Lapata(2007)Cohn and Lapata]{cohn-lapata-2007-machine}
Cohn, T. and Lapata, M.
\newblock Machine translation by triangulation: Making effective use of
  multi-parallel corpora.
\newblock In \emph{Proceedings of the 45th Annual Meeting of the Association of
  Computational Linguistics}, pp.\  728--735, Prague, Czech Republic, June
  2007. Association for Computational Linguistics.

\bibitem[Combes et~al.(2020)Combes, Zhao, Wang, and Gordon]{combes2020domain}
Combes, R. T.~d., Zhao, H., Wang, Y.-X., and Gordon, G.
\newblock Domain adaptation with conditional distribution matching and
  generalized label shift.
\newblock \emph{arXiv preprint arXiv:2003.04475}, 2020.

\bibitem[Conneau et~al.(2018{\natexlab{a}})Conneau, Lample, Ranzato, Denoyer,
  and J{\'{e}}gou]{conneau2018word}
Conneau, A., Lample, G., Ranzato, M., Denoyer, L., and J{\'{e}}gou, H.
\newblock Word translation without parallel data.
\newblock In \emph{Proceedings of the 6th International Conference on Learning
  Representations (ICLR 2018)}, 2018{\natexlab{a}}.

\bibitem[Conneau et~al.(2018{\natexlab{b}})Conneau, Rinott, Lample, Williams,
  Bowman, Schwenk, and Stoyanov]{conneau-etal-2018-xnli}
Conneau, A., Rinott, R., Lample, G., Williams, A., Bowman, S., Schwenk, H., and
  Stoyanov, V.
\newblock {XNLI}: Evaluating cross-lingual sentence representations.
\newblock In \emph{Proceedings of the 2018 Conference on Empirical Methods in
  Natural Language Processing}, pp.\  2475--2485, Brussels, Belgium,
  October-November 2018{\natexlab{b}}. Association for Computational
  Linguistics.
\newblock \doi{10.18653/v1/D18-1269}.

\bibitem[Conneau et~al.(2019)Conneau, Khandelwal, Goyal, Chaudhary, Wenzek,
  Guzm{\'a}n, Grave, Ott, Zettlemoyer, and Stoyanov]{conneau2019unsupervised}
Conneau, A., Khandelwal, K., Goyal, N., Chaudhary, V., Wenzek, G., Guzm{\'a}n,
  F., Grave, E., Ott, M., Zettlemoyer, L., and Stoyanov, V.
\newblock Unsupervised cross-lingual representation learning at scale.
\newblock \emph{arXiv preprint arXiv:1911.02116}, 2019.

\bibitem[De~Gispert \& Marino()De~Gispert and Marino]{de2006catalan}
De~Gispert, A. and Marino, J.~B.
\newblock Catalan-english statistical machine translation without parallel
  corpus: bridging through spanish.
\newblock Citeseer.

\bibitem[Devlin et~al.(2019)Devlin, Chang, Lee, and
  Toutanova]{devlin-etal-2019-bert}
Devlin, J., Chang, M.-W., Lee, K., and Toutanova, K.
\newblock {BERT}: Pre-training of deep bidirectional transformers for language
  understanding.
\newblock In \emph{Proceedings of the 2019 Conference of the North {A}merican
  Chapter of the Association for Computational Linguistics: Human Language
  Technologies, Volume 1 (Long and Short Papers)}, pp.\  4171--4186,
  Minneapolis, Minnesota, June 2019. Association for Computational Linguistics.
\newblock \doi{10.18653/v1/N19-1423}.

\bibitem[Dong et~al.(2015)Dong, Wu, He, Yu, and Wang]{dong-etal-2015-multi}
Dong, D., Wu, H., He, W., Yu, D., and Wang, H.
\newblock Multi-task learning for multiple language translation.
\newblock In \emph{Proceedings of the 53rd Annual Meeting of the Association
  for Computational Linguistics and the 7th International Joint Conference on
  Natural Language Processing (Volume 1: Long Papers)}, pp.\  1723--1732,
  Beijing, China, July 2015. Association for Computational Linguistics.
\newblock \doi{10.3115/v1/P15-1166}.

\bibitem[Faruqui \& Dyer(2014)Faruqui and Dyer]{faruqui-dyer-2014-improving}
Faruqui, M. and Dyer, C.
\newblock Improving vector space word representations using multilingual
  correlation.
\newblock In \emph{Proceedings of the 14th Conference of the {E}uropean Chapter
  of the Association for Computational Linguistics}, pp.\  462--471,
  Gothenburg, Sweden, April 2014. Association for Computational Linguistics.
\newblock \doi{10.3115/v1/E14-1049}.

\bibitem[Firat et~al.(2016{\natexlab{a}})Firat, Cho, and
  Bengio]{firat-etal-2016-multi}
Firat, O., Cho, K., and Bengio, Y.
\newblock Multi-way, multilingual neural machine translation with a shared
  attention mechanism.
\newblock In \emph{Proceedings of the 2016 Conference of the North {A}merican
  Chapter of the Association for Computational Linguistics: Human Language
  Technologies}, pp.\  866--875, San Diego, California, June
  2016{\natexlab{a}}. Association for Computational Linguistics.
\newblock \doi{10.18653/v1/N16-1101}.

\bibitem[Firat et~al.(2016{\natexlab{b}})Firat, Sankaran, Al-onaizan,
  Yarman~Vural, and Cho]{firat-etal-2016-zero}
Firat, O., Sankaran, B., Al-onaizan, Y., Yarman~Vural, F.~T., and Cho, K.
\newblock Zero-resource translation with multi-lingual neural machine
  translation.
\newblock In \emph{Proceedings of the 2016 Conference on Empirical Methods in
  Natural Language Processing}, pp.\  268--277, Austin, Texas, November
  2016{\natexlab{b}}. Association for Computational Linguistics.
\newblock \doi{10.18653/v1/D16-1026}.

\bibitem[Ganin et~al.(2016)Ganin, Ustinova, Ajakan, Germain, Larochelle,
  Laviolette, Marchand, and Lempitsky]{ganin2016domain}
Ganin, Y., Ustinova, E., Ajakan, H., Germain, P., Larochelle, H., Laviolette,
  F., Marchand, M., and Lempitsky, V.
\newblock Domain-adversarial training of neural networks.
\newblock \emph{The Journal of Machine Learning Research}, 17\penalty0
  (1):\penalty0 2096--2030, 2016.

\bibitem[Gouws et~al.(2015)Gouws, Bengio, and Corrado]{gouws2015bilbowa}
Gouws, S., Bengio, Y., and Corrado, G.
\newblock {BilBOWA}: Fast bilingual distributed representations without word
  alignments.
\newblock In \emph{Proceedings of ICML 2015}, pp.\  748--756, 2015.

\bibitem[Gu et~al.(2018)Gu, Hassan, Devlin, and Li]{gu2018universal}
Gu, J., Hassan, H., Devlin, J., and Li, V.~O.
\newblock Universal neural machine translation for extremely low resource
  languages.
\newblock In \emph{Proceedings of the 2018 Conference of the North American
  Chapter of the Association for Computational Linguistics: Human Language
  Technologies, Volume 1 (Long Papers)}, pp.\  344--354, 2018.

\bibitem[Guzm{\'a}n et~al.(2019)Guzm{\'a}n, Chen, Ott, Pino, Lample, Koehn,
  Chaudhary, and Ranzato]{guzman-etal-2019-flores}
Guzm{\'a}n, F., Chen, P.-J., Ott, M., Pino, J., Lample, G., Koehn, P.,
  Chaudhary, V., and Ranzato, M.
\newblock The {FLORES} evaluation datasets for low-resource machine
  translation: {N}epali{--}{E}nglish and {S}inhala{--}{E}nglish.
\newblock In \emph{Proceedings of the 2019 Conference on Empirical Methods in
  Natural Language Processing and the 9th International Joint Conference on
  Natural Language Processing (EMNLP-IJCNLP)}, pp.\  6098--6111, Hong Kong,
  China, November 2019. Association for Computational Linguistics.
\newblock \doi{10.18653/v1/D19-1632}.

\bibitem[Ha et~al.(2016)Ha, Niehues, and Waibel]{ha2016toward}
Ha, T.-L., Niehues, J., and Waibel, A.
\newblock Toward multilingual neural machine translation with universal encoder
  and decoder.
\newblock \emph{arXiv preprint arXiv:1611.04798}, 2016.

\bibitem[Huang et~al.(2019)Huang, Liang, Duan, Gong, Shou, Jiang, and
  Zhou]{huang-etal-2019-unicoder}
Huang, H., Liang, Y., Duan, N., Gong, M., Shou, L., Jiang, D., and Zhou, M.
\newblock {U}nicoder: A universal language encoder by pre-training with
  multiple cross-lingual tasks.
\newblock In \emph{Proceedings of the 2019 Conference on Empirical Methods in
  Natural Language Processing and the 9th International Joint Conference on
  Natural Language Processing (EMNLP-IJCNLP)}, pp.\  2485--2494, Hong Kong,
  China, November 2019. Association for Computational Linguistics.
\newblock \doi{10.18653/v1/D19-1252}.

\bibitem[Johansson et~al.(2016)Johansson, Shalit, and
  Sontag]{johansson2016learning}
Johansson, F., Shalit, U., and Sontag, D.
\newblock Learning representations for counterfactual inference.
\newblock In \emph{International conference on machine learning}, pp.\
  3020--3029, 2016.

\bibitem[Johnson et~al.(2017)Johnson, Schuster, Le, Krikun, Wu, Chen, Thorat,
  Vi{\'e}gas, Wattenberg, Corrado, et~al.]{johnson2017google}
Johnson, M., Schuster, M., Le, Q.~V., Krikun, M., Wu, Y., Chen, Z., Thorat, N.,
  Vi{\'e}gas, F., Wattenberg, M., Corrado, G., et~al.
\newblock Google’s multilingual neural machine translation system: Enabling
  zero-shot translation.
\newblock \emph{Transactions of the Association for Computational Linguistics},
  5:\penalty0 339--351, 2017.

\bibitem[Kiros et~al.(2015)Kiros, Zhu, Salakhutdinov, Zemel, Urtasun, Torralba,
  and Fidler]{kiros2015skip}
Kiros, R., Zhu, Y., Salakhutdinov, R.~R., Zemel, R., Urtasun, R., Torralba, A.,
  and Fidler, S.
\newblock Skip-thought vectors.
\newblock In \emph{Advances in neural information processing systems}, pp.\
  3294--3302, 2015.

\bibitem[Koehn et~al.(2003)Koehn, Och, and Marcu]{koehn-etal-2003-statistical}
Koehn, P., Och, F.~J., and Marcu, D.
\newblock Statistical phrase-based translation.
\newblock In \emph{Proceedings of the 2003 Human Language Technology Conference
  of the North {A}merican Chapter of the Association for Computational
  Linguistics}, pp.\  127--133, 2003.

\bibitem[Kondratyuk \& Straka(2019)Kondratyuk and
  Straka]{kondratyuk-straka-2019-75}
Kondratyuk, D. and Straka, M.
\newblock 75 languages, 1 model: Parsing universal dependencies universally.
\newblock In \emph{Proceedings of the 2019 Conference on Empirical Methods in
  Natural Language Processing and the 9th International Joint Conference on
  Natural Language Processing (EMNLP-IJCNLP)}, pp.\  2779--2795, Hong Kong,
  China, November 2019. Association for Computational Linguistics.
\newblock \doi{10.18653/v1/D19-1279}.

\bibitem[Lample \& Conneau(2019)Lample and Conneau]{lample2019cross}
Lample, G. and Conneau, A.
\newblock Cross-lingual language model pretraining.
\newblock \emph{Advances in Neural Information Processing Systems (NeurIPS)},
  2019.

\bibitem[Litschko et~al.(2018)Litschko, Glava{\v{s}}, Ponzetto, and
  Vuli{\'c}]{litschko2018unsupervised}
Litschko, R., Glava{\v{s}}, G., Ponzetto, S.~P., and Vuli{\'c}, I.
\newblock Unsupervised cross-lingual information retrieval using monolingual
  data only.
\newblock In \emph{The 41st International ACM SIGIR Conference on Research \&
  Development in Information Retrieval}, pp.\  1253--1256, 2018.

\bibitem[Luong et~al.(2015)Luong, Pham, and Manning]{luong2015bilingual}
Luong, T., Pham, H., and Manning, C.~D.
\newblock Bilingual word representations with monolingual quality in mind.
\newblock In \emph{Proceedings of the 1st Workshop on Vector Space Modeling for
  Natural Language Processing}, pp.\  151--159, 2015.

\bibitem[Mikolov et~al.(2013)Mikolov, Le, and Sutskever]{mikolov2013exploiting}
Mikolov, T., Le, Q.~V., and Sutskever, I.
\newblock Exploiting similarities among languages for machine translation.
\newblock \emph{arXiv preprint arXiv:1309.4168}, 2013.

\bibitem[Mirza \& Osindero(2014)Mirza and Osindero]{mirza2014conditional}
Mirza, M. and Osindero, S.
\newblock Conditional generative adversarial nets.
\newblock \emph{arXiv preprint arXiv:1411.1784}, 2014.

\bibitem[Neubig \& Hu(2018)Neubig and Hu]{neubig-hu-2018-rapid}
Neubig, G. and Hu, J.
\newblock Rapid adaptation of neural machine translation to new languages.
\newblock In \emph{Proceedings of the 2018 Conference on Empirical Methods in
  Natural Language Processing}, pp.\  875--880, Brussels, Belgium,
  October-November 2018. Association for Computational Linguistics.
\newblock \doi{10.18653/v1/D18-1103}.

\bibitem[Nie et~al.(1999)Nie, Simard, Isabelle, and Durand]{nie1999cross}
Nie, J.-Y., Simard, M., Isabelle, P., and Durand, R.
\newblock Cross-language information retrieval based on parallel texts and
  automatic mining of parallel texts from the web.
\newblock In \emph{Proceedings of the 22nd annual international ACM SIGIR
  conference on Research and development in information retrieval}, pp.\
  74--81, 1999.

\bibitem[Och \& Ney(2001)Och and Ney]{och2001statistical}
Och, F.~J. and Ney, H.
\newblock Statistical multi-source translation.
\newblock pp.\  253--258, 2001.

\bibitem[Pires et~al.(2019)Pires, Schlinger, and
  Garrette]{pires2019multilingual}
Pires, T., Schlinger, E., and Garrette, D.
\newblock How multilingual is multilingual bert?
\newblock \emph{arXiv preprint arXiv:1906.01502}, 2019.

\bibitem[Platanios et~al.(2018)Platanios, Sachan, Neubig, and
  Mitchell]{platanios18emnlp}
Platanios, E.~A., Sachan, M., Neubig, G., and Mitchell, T.
\newblock Contextual parameter generation for universal neural machine
  translation.
\newblock In \emph{Conference on Empirical Methods in Natural Language
  Processing (EMNLP)}, Brussels, Belgium, November 2018.

\bibitem[Ranzato et~al.(2019)Ranzato, Guzmán, Chaudhary, Chen, Shen, and
  Li]{fair2019umt}
Ranzato, M., Guzmán, P., Chaudhary, V., Chen, P.-J., Shen, J., and Li, X.
\newblock Recent advances in low-resource machine translation.
\newblock 2019.

\bibitem[Resnik(1999)]{resnik-1999-mining}
Resnik, P.
\newblock Mining the web for bilingual text.
\newblock In \emph{Proceedings of the 37th Annual Meeting of the Association
  for Computational Linguistics}, pp.\  527--534, College Park, Maryland, USA,
  June 1999. Association for Computational Linguistics.
\newblock \doi{10.3115/1034678.1034757}.

\bibitem[Resnik \& Smith(2003)Resnik and Smith]{resnik-smith-2003-web}
Resnik, P. and Smith, N.~A.
\newblock The web as a parallel corpus.
\newblock \emph{Computational Linguistics}, 29\penalty0 (3):\penalty0 349--380,
  2003.
\newblock \doi{10.1162/089120103322711578}.

\bibitem[Schwenk(2018)]{schwenk-2018-filtering}
Schwenk, H.
\newblock Filtering and mining parallel data in a joint multilingual space.
\newblock In \emph{Proceedings of the 56th Annual Meeting of the Association
  for Computational Linguistics (Volume 2: Short Papers)}, pp.\  228--234,
  Melbourne, Australia, July 2018. Association for Computational Linguistics.
\newblock \doi{10.18653/v1/P18-2037}.

\bibitem[Schwenk \& Douze(2017)Schwenk and Douze]{schwenk-douze-2017-learning}
Schwenk, H. and Douze, M.
\newblock Learning joint multilingual sentence representations with neural
  machine translation.
\newblock In \emph{Proceedings of the 2nd Workshop on Representation Learning
  for {NLP}}, pp.\  157--167, Vancouver, Canada, August 2017. Association for
  Computational Linguistics.
\newblock \doi{10.18653/v1/W17-2619}.

\bibitem[Shalit et~al.(2017)Shalit, Johansson, and
  Sontag]{shalit2017estimating}
Shalit, U., Johansson, F.~D., and Sontag, D.
\newblock Estimating individual treatment effect: generalization bounds and
  algorithms.
\newblock In \emph{Proceedings of the 34th International Conference on Machine
  Learning-Volume 70}, pp.\  3076--3085. JMLR. org, 2017.

\bibitem[Shen et~al.(2019)Shen, Chen, Le, He, Gu, Ott, Auli, and
  Ranzato]{shen2019source}
Shen, J., Chen, P.-J., Le, M., He, J., Gu, J., Ott, M., Auli, M., and Ranzato,
  M.
\newblock The source-target domain mismatch problem in machine translation.
\newblock \emph{arXiv preprint arXiv:1909.13151}, 2019.

\bibitem[Sutskever et~al.(2014)Sutskever, Vinyals, and Le]{NIPS2014_5346}
Sutskever, I., Vinyals, O., and Le, Q.~V.
\newblock Sequence to sequence learning with neural networks.
\newblock In Ghahramani, Z., Welling, M., Cortes, C., Lawrence, N.~D., and
  Weinberger, K.~Q. (eds.), \emph{Advances in Neural Information Processing
  Systems 27}, pp.\  3104--3112. Curran Associates, Inc., 2014.

\bibitem[Utiyama \& Isahara(2007)Utiyama and
  Isahara]{utiyama-isahara-2007-comparison}
Utiyama, M. and Isahara, H.
\newblock A comparison of pivot methods for phrase-based statistical machine
  translation.
\newblock In \emph{Human Language Technologies 2007: The Conference of the
  North {A}merican Chapter of the Association for Computational Linguistics;
  Proceedings of the Main Conference}, pp.\  484--491, Rochester, New York,
  April 2007. Association for Computational Linguistics.

\bibitem[Vincent et~al.(2008)Vincent, Larochelle, Bengio, and
  Manzagol]{vincent2008extracting}
Vincent, P., Larochelle, H., Bengio, Y., and Manzagol, P.-A.
\newblock Extracting and composing robust features with denoising autoencoders.
\newblock In \emph{Proceedings of the 25th international conference on Machine
  learning}, pp.\  1096--1103, 2008.

\bibitem[Wang et~al.(2017)Wang, Finch, Utiyama, and Sumita]{wang2017sentence}
Wang, R., Finch, A., Utiyama, M., and Sumita, E.
\newblock Sentence embedding for neural machine translation domain adaptation.
\newblock In \emph{Proceedings of the 55th Annual Meeting of the Association
  for Computational Linguistics (Volume 2: Short Papers)}, pp.\  560--566,
  2017.

\bibitem[Wu et~al.(2016)Wu, Schuster, Chen, Le, Norouzi, Macherey, Krikun, Cao,
  Gao, Macherey, Klingner, Shah, Johnson, Liu, Łukasz Kaiser, Gouws, Kato,
  Kudo, Kazawa, Stevens, Kurian, Patil, Wang, Young, Smith, Riesa, Rudnick,
  Vinyals, Corrado, Hughes, and Dean]{gnmt}
Wu, Y., Schuster, M., Chen, Z., Le, Q.~V., Norouzi, M., Macherey, W., Krikun,
  M., Cao, Y., Gao, Q., Macherey, K., Klingner, J., Shah, A., Johnson, M., Liu,
  X., Łukasz Kaiser, Gouws, S., Kato, Y., Kudo, T., Kazawa, H., Stevens, K.,
  Kurian, G., Patil, N., Wang, W., Young, C., Smith, J., Riesa, J., Rudnick,
  A., Vinyals, O., Corrado, G., Hughes, M., and Dean, J.
\newblock Google's neural machine translation system: Bridging the gap between
  human and machine translation.
\newblock \emph{CoRR}, abs/1609.08144, 2016.

\bibitem[Zemel et~al.(2013)Zemel, Wu, Swersky, Pitassi, and
  Dwork]{zemel2013learning}
Zemel, R., Wu, Y., Swersky, K., Pitassi, T., and Dwork, C.
\newblock Learning fair representations.
\newblock In \emph{International Conference on Machine Learning}, pp.\
  325--333, 2013.

\bibitem[Zhang et~al.(2018)Zhang, Lemoine, and Mitchell]{zhang2018mitigating}
Zhang, B.~H., Lemoine, B., and Mitchell, M.
\newblock Mitigating unwanted biases with adversarial learning.
\newblock In \emph{Proceedings of the 2018 AAAI/ACM Conference on AI, Ethics,
  and Society}, pp.\  335--340, 2018.

\bibitem[Zhao \& Gordon(2019)Zhao and Gordon]{zhao2019inherent}
Zhao, H. and Gordon, G.
\newblock Inherent tradeoffs in learning fair representations.
\newblock In \emph{Advances in neural information processing systems}, pp.\
  15649--15659, 2019.

\bibitem[Zhao et~al.(2015)Zhao, Lu, and Poupart]{zhao2015self}
Zhao, H., Lu, Z., and Poupart, P.
\newblock Self-adaptive hierarchical sentence model.
\newblock In \emph{Twenty-fourth international joint conference on artificial
  intelligence}, 2015.

\bibitem[Zhao et~al.(2018)Zhao, Zhang, Wu, Moura, Costeira, and
  Gordon]{zhao2018adversarial}
Zhao, H., Zhang, S., Wu, G., Moura, J.~M., Costeira, J.~P., and Gordon, G.~J.
\newblock Adversarial multiple source domain adaptation.
\newblock In \emph{Advances in neural information processing systems}, pp.\
  8559--8570, 2018.

\bibitem[Zhao et~al.(2019{\natexlab{a}})Zhao, Coston, Adel, and
  Gordon]{zhao2019conditional}
Zhao, H., Coston, A., Adel, T., and Gordon, G.~J.
\newblock Conditional learning of fair representations.
\newblock \emph{arXiv preprint arXiv:1910.07162}, 2019{\natexlab{a}}.

\bibitem[Zhao et~al.(2019{\natexlab{b}})Zhao, Des~Combes, Zhang, and
  Gordon]{zhao2019learning}
Zhao, H., Des~Combes, R.~T., Zhang, K., and Gordon, G.
\newblock On learning invariant representations for domain adaptation.
\newblock In \emph{International Conference on Machine Learning}, pp.\
  7523--7532, 2019{\natexlab{b}}.

\bibitem[Zoph \& Knight(2016)Zoph and Knight]{zoph-knight-2016-multi}
Zoph, B. and Knight, K.
\newblock Multi-source neural translation.
\newblock In \emph{Proceedings of the 2016 Conference of the North {A}merican
  Chapter of the Association for Computational Linguistics: Human Language
  Technologies}, pp.\  30--34, San Diego, California, June 2016. Association
  for Computational Linguistics.
\newblock \doi{10.18653/v1/N16-1004}.

\end{thebibliography}
\bibliographystyle{icml2020}

\newpage
\appendix

\section{Missing Proofs in Section~\ref{sec:positive}}
In this section we provide all the missing proofs in Section~\ref{sec:positive}. Again, in what follows we will first restate the corresponding theorems for the ease of reading and then provide the detailed proofs. 
\ggap*
\begin{proof}
For $f\in\Fcal$, define $\ell_S(f)\defeq \risk(f) - \emprisk_S(f)$ to be the generalization error of $f$ on sample $S$. The first step is to prove the following inequality holds for $\forall f_1, f_2\in\Fcal$ and any sample $S$:
\begin{equation*}
    |\ell_S(f_1) - \ell_S(f_2)| \leq 8M\cdot\|f_1 - f_2\|_\infty.
\end{equation*}
In other words, $\ell_S(\cdot)$ is a Lipschitz function in $\Fcal$ w.r.t.\ the $\ell_\infty$ norm. To see, by definition of the generalization error, we have
\begin{align*}
    &|\ell_S(f_1) - \ell_S(f_2)| \\
    &= |\risk(f_1) - \emprisk_S(f_1) - \risk(f_2) + \emprisk_S(f_2)| \\
    &\leq |\risk(f_1) - \risk(f_2)| + |\emprisk_S(f_1) - \emprisk_S(f_2)|.
\end{align*}
To get the desired upper bound, it suffices for us to bound $|\risk(f_1) - \risk(f_2)|$ by $\|f_1 - f_2\|_\infty$ and the same technique could be used to upper bound $|\emprisk_S(f_1) - \emprisk_S(f_2)|$ since the only difference lies in the measure where the expectation is taken over. We now proceed to upper bound $|\risk(f_1) - \risk(f_2)|$:
\begin{align*}
    |\risk(f_1) - \risk(f_2)| &= \big|\Exp_{\zz\sim\dist}[\|f_1(\xx) - \xx'\|_2^2] - \Exp_{\zz\sim\dist}[\|f_2(\xx) - \xx'\|_2^2]\big| \\
    &= \big|\Exp_{\zz\sim\dist}[\|f_1(\xx)\|_2^2 - \|f_2(\xx)\|_2^2 - 2\xx'^T(f_1(\xx) - f_2(\xx))]\big| \\
    &\leq \Exp_{\zz\sim\dist}\big|(f_1(\xx) - f_2(\xx))^T(f_1(\xx) + f_2(\xx)) - 2\xx'^T(f_1(\xx) - f_2(\xx))\big| \\
    &\leq \Exp_{\zz\sim\dist}\left[\big|(f_1(\xx) - f_2(\xx))^T(f_1(\xx) + f_2(\xx))\big|\right] +2\Exp_{\zz\sim\dist}\left[\big|\xx'^T(f_1(\xx) - f_2(\xx))\big|\right] \\
    &\leq \Exp_{\zz\sim\dist}\left[\|f_1(\xx) - f_2(\xx)\|\cdot\|f_1(\xx) + f_2(\xx)\|\right] + 2\Exp_{\zz\sim\dist}\left[\|\xx'\|\cdot\|f_1(\xx) - f_2(\xx)\|\right] \\
    &\leq 2M\Exp_{\zz\sim\dist}\left[\|f_1(\xx) - f_2(\xx)\|\right] + 2M\Exp_{\zz\sim\dist}\left[\|f_1(\xx) - f_2(\xx)\|\right] \\
    &\leq  4M\|f_1 - f_2\|_\infty.
\end{align*}
In the proof above, the first inequality holds due to the monotonicity property of integral. The second inequality holds by triangle inequality. The third one is due to Cauchy-Schwarz inequality. The fourth inequality holds by the assumption that $\forall f\in\Fcal$, $\max_{x\in\xxspace}\|f(\xx)\|\leq M$ and the identity mapping is in $\Fcal$ so that $\|\xx'\| = \|\id(\xx')\| \leq \|\id(\cdot)\|_\infty \leq M$. The last one holds due to the monotonicity property of integral. 

It is easy to see that the same argument could also be used to show that $|\emprisk_S(f_1) - \emprisk_S(f_2)|\leq 4M\|f_1 - f_2\|_\infty$. Combine these two inequalities, we have
\begin{align*}
    |\ell_S(f_1) - \ell_S(f_2)| &\leq |\risk(f_1) - \risk(f_2)| + |\emprisk_S(f_1) - \emprisk_S(f_2)| \\
    &\leq  8M\|f_1 - f_2\|_\infty.
\end{align*}

In the next step, we show that suppose $\Fcal$ could be covered by $k$ subsets $\Ccal_1,\ldots,\Ccal_k$, i.e., $\Fcal = \cup_{i\in[k]}\Ccal_i$. Then for any $\epsilon > 0$, the following upper bound holds:
\begin{equation*}
    \Pr_{S\sim\dist^n}\big(\sup_{f\in\Fcal}|\ell_S(f)| \geq \epsilon\big) \leq \sum_{i\in[k]}\Pr_{S\sim\dist^n}\big(\sup_{f\in\Ccal_i}|\ell_S(f)| \geq \epsilon\big).
\end{equation*}
This follows from the union bound:
\begin{align*}
    \Pr_{S\sim\dist^n}\big(\sup_{f\in\Fcal}|\ell_S(f)| \geq \epsilon\big) &= \Pr_{S\sim\dist^n}\big(\bigcup_{i\in[k]}\sup_{f\in\Ccal_i}|\ell_S(f)| \geq \epsilon \big) \\
    &\leq \sum_{i\in[k]}\Pr_{S\sim\dist^n}\big(\sup_{f\in\Ccal_i}|\ell_S(f)| \geq \epsilon\big).
\end{align*}
Next, within each $L_\infty$ ball $\Ccal_i$ centered at $f_i$ with radius $\frac{\epsilon}{16M}$ such that $\Fcal\subseteq \cup_{i\in[k]}\Ccal_i$, we bound each term in the above union bound as:
\begin{equation*}
    \Pr_{S\sim\dist^n}\big(\sup_{f\in\Ccal_i}|\ell_S(f)| \geq \epsilon\big) \leq \Pr_{S\sim\dist^n}\big(|\ell_S(f_i)| \geq \epsilon / 2\big).
\end{equation*}
To see this, realize that $\forall f\in\Ccal_i$, we have $\|f - f_i\|_\infty \leq \epsilon/16M$, which implies
\begin{equation*}
    |\ell_S(f) - \ell_S(f_i)| \leq 8M\|f - f_i\|_\infty \leq \frac{\epsilon}{2}.
\end{equation*}
Hence we must have $|\ell_S(f_i)| \geq \epsilon / 2$, otherwise $\sup_{f\in\Ccal_i}|\ell_S(f)| < \epsilon$. This argument means that 
\begin{equation*}
\Pr_{S\sim\dist^n}\big(\sup_{f\in\Ccal_i}|\ell_S(f)| \geq \epsilon\big) \leq \Pr_{S\sim\dist^n}\big(|\ell_S(f_i)| \geq \epsilon / 2\big).    
\end{equation*}
To finish the proof, we use the standard Hoeffding inequality to upper bound $\Pr_{S\sim\dist^n}\big(|\ell_S(f_i)| \geq \epsilon / 2\big)$ as follows:
\begin{align*}
    \Pr_{S\sim\dist^n}\big(|\ell_S(f_i)| \geq \epsilon / 2\big) &= \Pr_{S\sim\dist^n}\big(|\risk(f_i) - \emprisk_S(f_i)| \geq \epsilon / 2\big) \\
    &\leq 2\exp\left(-\frac{2n^2(\epsilon/2)^2}{n((2M)^2 - 0)^2}\right) \\
    &= 2\exp\left(-\frac{n\epsilon^2}{16M^4}\right).
\end{align*}
Now combine everything together, we obtain the desired upper bound as stated in the lemma.
\begin{equation*}
    \Pr_{S\sim \dist^n}\left(\sup_{f\in\Fcal} |\risk(f) - \emprisk_S(f)| \geq \epsilon\right) \leq 2\covering(\Fcal, \frac{\epsilon}{16M})\cdot\exp\left(\frac{-n\epsilon^2}{16M^4}\right).\qedhere
\end{equation*}
\end{proof}

We next prove the generalization bound for a single pair of translation task:
\gsingle*
\begin{proof}
This is a direct corollary of Lemma~\ref{thm:covering} by setting the upper bound in Lemma~\ref{thm:covering} to be $\delta$ and solve for $\epsilon$. 
\end{proof}

We now provide the proof sketch of Theorem~\ref{thm:random}. The main proof idea is exactly the same as the one we have in the deterministic setting, except that we replace the original definitions of errors and Lipschitzness with the generalized definitions under the randomized setting. 
\random*
\begin{proof}[Proof Sketch]
The first step is prove the corresponding error concentration lemma using covering numbers as the one in Lemma~\ref{thm:covering}. Again, due to the assumption that $\Fcal$ is closed under composition, we have $D_{L'}\circ E_{L}\in\Fcal$, hence it suffices if we could prove a uniform convergence bound for an arbitrary function $f\in\Fcal$. To this end, for $f\in\Fcal$, define $\ell_S(f)\defeq \risk(f) - \emprisk_S(f)$ to be the generalization error of $f$ on sample $S$. The first step is to prove the following inequality holds for $\forall f_1, f_2\in\Fcal$ and any sample $S$:
\begin{equation*}
    |\ell_S(f_1) - \ell_S(f_2)| \leq 8M\cdot\|f_1 - f_2\|_\infty.
\end{equation*}
In other words, $\ell_S(\cdot)$ is a Lipschitz function in $\Fcal$ w.r.t.\ the $\ell_\infty$ norm. To see this, by definition of the generalization error, we have
\begin{align*}
    |\ell_S(f_1) - \ell_S(f_2)| = |\risk(f_1) - \emprisk_S(f_1) - \risk(f_2) + \emprisk_S(f_2)| \leq |\risk(f_1) - \risk(f_2)| + |\emprisk_S(f_1) - \emprisk_S(f_2)|.
\end{align*}
To get the desired upper bound, it suffices for us to bound $|\risk(f_1) - \risk(f_2)|$ by $\|f_1 - f_2\|_\infty$ and the same technique could be used to upper bound $|\emprisk_S(f_1) - \emprisk_S(f_2)|$ since the only difference lies in the measure where the expectation is taken over. 

Before we proceed, in order to make the notation uncluttered, we first simplify $\risk(f)$:
\begin{align*}
    \risk(f) &= \Exp_{r,r'}\left[\|f - \dec{L'}\circ \enc{L}\|_{\ell_2({\dec{L}}_\sharp(\dist\times \dist_r))}^2\right].
\end{align*}
Define $\zz\sim\dist$ to mean the sampling process of $(x, r, r')\sim {\dec{L}}_\sharp(\dist\times \dist_r)\times D_r\times D_{r'}$, $\xx\defeq (x, r, r')$ and $\xx'\defeq \dec{L'}(\enc{L}(x, r'), r)$. Then 
\begin{align*}
    \risk(f) &= \Exp_{r,r'}\left[\|f - \dec{L'}\circ \enc{L}\|_{\ell_2({\dec{L}}_\sharp(\dist\times \dist_r))}^2\right] \\
             &= \Exp_{\zz\sim\dist}[\|f(\xx) - \xx'\|_2^2].
\end{align*}
With the simplified notation, it is now clear that we essentially reduce the problem in the randomized setting to the original one in the deterministic setting. Hence by using exactly the same proof as the one of Lemma~\ref{thm:covering}, we can obtain the following high probability bound:
\begin{equation*}
    \Pr\left(\sup_{f\in\Fcal} |\risk(f) - \emprisk_S(f)| \geq \epsilon\right) \leq 2\covering(\Fcal, \frac{\epsilon}{16M})\cdot\exp\left(\frac{-n\epsilon^2}{16M^4}\right).
\end{equation*}
As a direct corollary, a similar generalization bound for a single pair of translation task like the one in Theorem~\ref{thm:finite} also holds. To finish the proof, by the linearity of the expectation $\Exp_{r, r'}$, it is clear that exactly the same chaining argument in the proof of Theorem~\ref{thm:main} could be used as well as the only thing we need to do is to take an additional expectation $\Exp_{r,r'}$ at the most outside level.
\end{proof}

\end{document}